\documentclass{IEEEtran}
\usepackage{cite}
\usepackage{color}
\usepackage{graphicx}
\usepackage{epstopdf}
\usepackage[cmex10]{amsmath}
\usepackage{amssymb}
\usepackage{algorithm,algorithmicx,algpseudocode}

\usepackage{amsmath,amssymb}
\usepackage{lipsum}
\usepackage{comment}
\usepackage{array}
\input{mysymbol.sty}
\usepackage{needspace}

\usepackage{amsthm,bm}
\usepackage{bbm}
\usepackage{url}

\newtheorem{proposition}{Proposition}

\newtheorem{theorem}{Theorem}
\newtheorem{definition}{Definition}
\newtheorem{lemma}{Lemma}

\theoremstyle{definition}
\newtheorem{assumption}{Assumption}

\newcommand{\sigmah}{\Sigma_\ccalH}

\author{Santiago Paternain$^\dagger$, Juan Andr\'es Bazerque$^*$ and Alejandro Ribeiro$^\ddagger$
  \thanks{Work supported by ARL DCIST CRA W911NF-17-2-0181. $^\dagger$  Dept. of Electrical and System Engineering, Univ. of Pennsylvania. Email: \{spater,aribeiro\} @seas.upenn.edu. $^*$Univ. de la Rep\'ublica. Email:jbazerque@fing.edu.uy
}}

\renewcommand{\comment}[1]{}
\newcolumntype{S}{>{\centering\arraybackslash} m{.10\linewidth} }
\newcolumntype{T}{>{\centering\arraybackslash} m{.30\linewidth} }
\title{ Policy Gradient for Continuing Tasks in Non-stationary Markov Decision Processes }
\begin{document}

\maketitle


%
\begin{abstract}
  Reinforcement learning considers the problem of finding policies that maximize an expected cumulative reward in a Markov decision process with unknown transition probabilities. In this paper we consider the problem of finding optimal policies assuming that they belong to a reproducing kernel Hilbert space (RKHS). To that end we compute unbiased stochastic gradients of the value function which we use as ascent directions to update the policy. {A major drawback of policy gradient-type algorithms is that they are limited to episodic tasks unless stationarity assumptions are imposed. Hence preventing these algorithms to be fully implemented online, which is a desirable property for systems that need to adapt to new tasks and/or environments in deployment.} The main requirement for a policy gradient algorithm to work is that the estimate of the gradient at any point in time is an ascent direction for the initial value function. In this work we establish that indeed this is the case which enables to show the convergence of the online algorithm to the critical points of the initial value function. {A numerical example shows the ability of our online algorithm to learn to solve a navigation and surveillance problem, in which an agent must loop between to  goal locations. This example corroborates our theoretical findings about the ascent directions of subsequent stochastic gradients. It also shows how the agent running our online algorithm succeeds in learning to  navigate, following a continuing cyclic trajectory that does not comply with the standard stationarity assumptions in the literature for non episodic training.}
\end{abstract}

%
\section{Introduction}\label{sec_intro}


Reinforcement learning (RL) problems--which is the interest in this paper--are a special setting for the analysis of Markov decision processes (MDPs) in which both the transition probabilities are unknown. The agent interacts with the environment and observes samples of a reward random variable associated to a given state and action pair  \cite{sutton1998reinforcement}. These rewards samples are used to update the policy of the agent so as to maximize the Q- function, defined as the expected cumulative reward conditioned to the current state and action. The solutions to RL problems are divided in two main approaches. On one hand, those approaches that aim to learn the Q-function to then choose the action that for the current state maximizes said function. Among these algorithms the standard solution is Q-learning \cite{watkins1992q}, whose earlier formulations were applicable in scenarios where the state and the action are discrete. The aforementioned algorithms suffer from the \emph{curse of dimensionality}, with complexity growing exponentially with the number of actions and states \cite{friedman2001elements}. This is of particular concern in problems where the state and the actions are continuous spaces, and thus, any reasonable discretization leads to a large number of states and possible actions. A common approach to overcome this difficulty is to assume that the Q-function admits a finite parameterization that can be linear \cite{sutton2009convergent}, rely on a nonlinear basis expansion \cite{bhatnagar2009convergent}, or be given by a neural network \cite{mnih2013playing}. Alternatively one can assume that the Q-function \cite{koppel2017breaking, tolstayanonparametric} belongs to a reproducing kernel Hilbert space. However, in these cases, maximizing the Q-function to select the best possible action is computationally challenging. Moreover, when using function approximations Q-learning may diverge \cite{baird1995residual}.

This motivates the development of another class of algorithms that attempts to learn the optimal policy by running stochastic gradient ascent on the Q-function with respect to the policy parameters \cite{williams1992simple,sutton2000policy, deisenroth2013survey} or with respect to the policy itself in the case of non-parametric representations \cite{lever2015modelling,paternain2018stochastic}. These gradients involve the computation of expectations which requires knowledge of the underlying probabilistic model. With the goal of learning from data only, they provide unbiased estimates of the gradients which are used for  stochastic approximation \cite{robbins1951stochastic}. One of the classic examples of estimates of the gradients used in discrete state-action spaces is REINFORCE \cite{williams1992simple}. Similar estimates can also be computed in the case of parametric \cite{baxter2001infinite} and non-parametric function approximations \cite{paternain2018stochastic}. Once these unbiased estimates have been computed convergence to the critical points can be established under a diminishing step-size as in the case in parametric optimization \cite{bertsekas1999nonlinear}. A drawback of said estimators is that they have high variance and therefore they suffer from slow convergence. This issue can be mitigated using Actor-Critic methods \cite{konda2000actor,bhatnagar2009natural,degris2012off} to estimate the policy gradients. To compute these estimates however, one is required to re-initialize the system for every new iteration. Hence, limiting its application to episodic tasks \cite{williams1992simple, paternain2018stochastic}.



A common workaround to this hurdle is to modify the value function so to consider the average rate of reward instead of the cumulative reward \cite[Chapter 13]{sutton1998reinforcement}. This formulation also requires that under every policy the MDP converges to a steady state distribution that is independent of the initial state. The convergence to a stationary distribution is restrictive in many cases, as we discus in Section \ref{sec_problem_formulation}, since it prevents the agents from considering policies that result in cyclic behaviors for instance. Not being able to reproduce these behaviors is a drawback for problems like surveillance where the policy that the agent should follow is one that visits  different points of interest. Moreover, even for problems where the target is a specific state, and thus the convergence to the stationary distribution is a reasonable assumption, the average reward formulation may modify the optimal policies in the sense that it is a formulation that ignores transient behaviors. We discuss this issue in more detail also in Section \ref{sec_problem_formulation}.

Given the  offline policy gradient algorithm in \cite{paternain2018stochastic}, we aim to avoid the reinitialization  requirement while keeping the cumulative reward as value function. In particular, we compute stochastic gradients as in \cite{paternain2018stochastic} which are guaranteed to be unbiased estimates of the gradient of the value function at the state that systems finds itself at the beginning of the iteration. Because this estimation  requires rollouts at  each iteration, the agent is in fact computing estimates of the gradients of  value functions at different states.


Building on our preliminary results \cite{paternain2019cdc}, we  establish in Theorem \ref{prop_all_gradients} that the gradients of the value function at any state are also  ascent directions of the value function at the initial state. Leveraging this result we address the convergence of the online policy gradient algorithm to a neighborhood of the critical points in Theorem \ref{theo_convergence}, hence dropping the assumption of the convergence to\textemdash and existence of\textemdash the stationary distribution over states for every intermediate policy. These results are backed by Proposition \ref{prop_common_critical_points}, which establishes that a critical point of the value function conditioned at the initial state is also a critical point for the value functions conditioned at states in the future, suggesting that the landscape of different value functions is still very similar.  Finally, as an accessory computational refinement, we add a compression step to the online algorithm to reduce the number of kernels by trading-off a discretionary convergence error. This refinement uses Orthogonal Kernel Matching Pursuit \cite{koppel2016parsimonious}. Other than concluding remarks the paper ends with numerical experiments {where we consider an agent whose goal is to surveil a region of the space while having to visit often enough a charging station. } The cyclic nature of this problem evidences the ability of our algorithm to operate in a non stationary setup and carry the task by training in a fully online fashion, without the need of episodic restarts. The experiment is also useful to corroborate our theoretical findings about the ascending direction properties of stochastic gradients computed at different points of the trajectory.


\section{Problem Formulation}\label{sec_problem_formulation}
In this work we are interested in the problem of finding a policy that maximizes the expected discounted cumulative reward of an agent that chooses actions sequentially. Formally, let us denote the time by $t \in \left\{\{0\},\mathbb{N} \right\}$ and let $\ccalS \subset \mathbb{R}^n$ be a compact set denoting the state space of the agent, and $\ccalA = \mathbb{R}^p$ be its action space. The transition dynamics are governed by a conditional probability $ P_{s_t\to s_{t+1}}^{a_t}(s) :=p(s_{t+1}=s|(s_t,a_t) \in \ccalS\times \ccalA)$ satisfying the Markov property, i.e.,
%
$p(s_{t+1}=s\big|(s_u,a_u) \in \ccalS\times \ccalA, \forall u\leq t) 
    =p(s_{t+1}=s|(s_t,a_t) \in \ccalS\times \ccalA).$
    %
The policy of the agent is a multivariate Gaussian distribution with mean $h: \ccalS \to \ccalA$. The later map is assumed to be a vector-valued function in a vector-valued RKHS $\ccalH$. We formally define this notion next, with comments ensuing.
%
\begin{definition}\label{def_rkhs}
  A vector valued RKHS $\ccalH$ is a Hilbert space of functions $h:\ccalS\to \mathbb{R}^p$
     such that for all $\bbc\in\mathbb{R}^p$ and $s \in \ccalS$, the following reproducing property holds
  \begin{equation}\label{eqn_reproducing}
<h(\cdot),\kappa(s,\cdot)\bbc>_{\ccalH} = h(s)^\top\bbc.
  \end{equation}
 where
   $\kappa(s,s^\prime)$ is a symmetric matrix-valued function that renders a positive definite matrix when evaluated at  any $s,s^\prime \in \ccalS$. \end{definition}
  
%
%
If $\kappa(s,s^\prime)$ is a diagonal matrix-valued function with diagonal elements  $\kappa(s,s^\prime)_{ii}$,  and $\mathbf c$ is the $i$-th canonical vector in $\mathbb R^p$, then \eqref{eqn_reproducing} reduces to the standard one-dimensional reproducing property per coordinate 
%
  %
  %
$h_i(s) = <h_i(\cdot),\kappa(s,\cdot)_{ii}>.$ 
  %
  %
%
 With the above definitions the policy of the agent is the following conditional probability of the action $\pi_h(a|s):\ccalS\times \ccalA \to \mathbb{R}_+$, with
\begin{equation}\label{eqn_gaussian_policy}
\pi_h(a|s) = \frac{1}{\sqrt{\det(2\pi\Sigma)}}e^{-\frac{(a-h(s))^\top\Sigma^{-1}(a-h(s))}{2}}.
\end{equation}
The latter means that given a function $h\in\ccalH$ and the current state $s\in\ccalS$, the agent selects an action $a\in\ccalA$ from a multivariate normal distribution $\ccalN(h(s),\Sigma)$. The choice of a random policy over a deterministic policy $a=h(s)$ makes the problem tractable both theoretically and numerically as it is explained in \cite{paternain2018stochastic}. The actions selected by the agent result in a reward defined by a function $r:\ccalS\times\ccalA \to \mathbb{R}$. 

The objective is then to find a policy $h^\star \in \ccalH$ such that it maximizes the expected discounted reward 
\begin{equation}\label{eqn_problem_statement}
  h^\star := \argmax_{h\in\ccalH} U_{s_0}(h) = \argmax_{h\in\ccalH} \mathbb{E}\left[\sum_{t=0}^\infty \gamma^t r(s_t,a_t)\Big| h,s_0\right] , 
  \end{equation}
where the expectation is taken with respect to all states except $s_0$, i.e., $s_1,\ldots $ and all actions $a_0,a_1,\ldots,$. The parameter $\gamma \in(0,1)$ is a discount factor that gives relative weights to the reward at different times. Values of $\gamma$ close to one imply that current rewards  are as important as future rewards, whereas smaller values of $\gamma$ give origin to myopic policies that prioritize maximizing immediate rewards. It is also noticeable that $U_{s_0}(h)$ is indeed a function of the policy $h$, since policies affect the trajectories $\{s_t,a_t\}_{t=0}^\infty$. 
%

As discussed in Section \ref{sec_intro} problem \eqref{eqn_problem_statement} can be tackled using methods of the policy gradient type \cite[Chapter 13]{sutton1998reinforcement}. These methods have been extended as well to non-parametric scenarios as we consider here \cite{lever2015modelling,paternain2018stochastic}. A drawback of these methods is that they require restarts which prevents them from a fully online implementation. To better explain this claim let us write down the expression of the gradient of the objective in \eqref{eqn_problem_statement} with respect to $h$.  Before doing so, we are required to define the discounted long-run probability distribution $\rho_{s_0}(s,a)$ 
\begin{align}\label{eqn_discounted_distribution}
      &\hspace{.3cm}\rho_{s_0}(s,a) := (1-\gamma)\sum_{t=0}^\infty \gamma^t p(s_t=s,a_t=a|s_0), 
    \end{align} 
where $p(s_t=s,a_t=a|s_0)$ is the distribution of the MDP under a policy $h$
\begin{equation}\label{eqn_mdp_distribution}
\begin{split}
  p(s_t=s,a_t=a|s_0)= \\
  \pi_h(a_t|s_t) \int \prod_{u=0}^{t-1} p(s_{u+1}|s_u,a_u)\pi_h(a_u|s_u) \, d\bbs_{t-1} d\bba_{t-1},
  \end{split}
\end{equation}
with $d\bbs_{t-1} = (ds_1,\ldots ds_{t-1})$ and $d\bba_{t-1} = (da_0,\ldots da_{t-1})$. We also require to define $Q(s,a;h)$, the expected discounted reward for a policy $h$ that at state $s$ selects action $a$. Formally this is
\begin{equation}\label{eqn_q_function}
Q(s,a;h) := \mathbb{E}\left[\sum_{t=0}^\infty \gamma^t r(s_t,a_t)\Big| h, s_0=s, a_0=a\right].
\end{equation}
With these functions defined, the gradient of the discounted rewards with respect to $h$ yields \cite{sutton2000policy,lever2015modelling} 
\begin{align}\label{eqn_nabla_U}
     &\nabla_h U_{s_0}(h,\cdot)= \\
&\frac{1}{1-\gamma}\mathbb{E}_{(s,a)\sim \rho_{s_0}(s,a)}\left[ Q(s,a;h)\kappa(s,\cdot)\Sigma^{-1}\left(a-h(s)\right) \Big| h\right],
 \nonumber   \end{align}
where the gradient of $U_{s_0}(h)$ with respect to the continuous function $h$ is defined in the sense of Frechet, rendering a function in $\mathcal H$. This is represented by the notation in $\nabla_h U_{s_0}(h,\cdot)$, where the dot substitutes the second variable of the kernel, belonging to $\mathcal S$, which is omitted to simplify notation. Observe that the expectation with respect to the distribution $\rho_{s_0}(s,a)$ is an integral of an infinite sum over a continuous space. Although this could have a tractable solution in some specific cases, this would require the system transition density $p(s_{t+1}|s_t,a_t)$ which is unknown in the context of RL. Thus, computing \eqref{eqn_nabla_U} in closed form becomes impractical. In fact, a large number of samples might be needed to obtain an accurate Monte Carlo approximation even if $p(s_{t+1}|s_t,a_t)$ were known. In \cite{paternain2018stochastic} an offline stochastic gradient ascent algorithm is proposed to overcome these difficulties and it is shown to converge to a critical point of the functional $U_{s_0}$. {Notice that to compute stochastic approximations of the gradient \eqref{eqn_nabla_U} on is required to sample from the distribution $\rho_{s_0}$ which depends on the initial condition. This dependency results in a fundamental limitation for online implementation.} We present in Section \ref{sec_offline_alg} the algorithm and a summary of the main results in \cite{paternain2018stochastic} since it serves as the basis for a fully online algorithm and to understanding the aforementioned difficulties associated to the online problem in detail. {Before doing so we discuss a common workout to the continuing task problem\textemdash or the problem of avoiding restarts. 
}
%
%
{
\subsection{Reinforcement learning in continuous tasks}\label{sec_continuing}

  When considering continuing tasks it is customary to modify the objective \eqref{eqn_problem_statement} and instead attempt to maximize the undiscounted objective \cite[Chapter 13]{sutton2018reinforcement}
  \begin{equation}\label{eqn_average_return}
U_{s_0}^\prime(h) = \lim_{T\to \infty} \frac{1}{T}\sum_{t=1}^T\mathbb{E}\left[r(s_t,a_t)\Big|  s_0\right]. 
  \end{equation}
  Consider a steady state  distribution  $\rho^\prime(s)$ that is ergodic and independent of the starting point, this is, a distribution that satisfies
  \begin{equation}\label{eqn_ergodic_distribution}
\rho^\prime(s^\prime) = \int \rho^\prime(s)\pi_h(a|s)p(s^\prime|s,a) \, ds da,
    \end{equation}
  for all $s^\prime\in\ccalS$. Under the assumptions that such distribution exists and that the distribution of the MDP converges to it, the limit in \eqref{eqn_average_return} is finite. Then, using Stolz-Cesaro's Lemma (see e.g. \cite[pp. 85-88]{muresan2009concrete}), then \eqref{eqn_average_return} reduces to
  \begin{equation}
U_{s_0}^\prime(h) = \lim_{t\to \infty} \mathbb{E}\left[r(s_t,a_t)\Big|  s_0\right],
  \end{equation}
  where the expectation is with respect to the stationary distribution~\eqref{eqn_ergodic_distribution}. Thus, we can rewrite the previous expression as
  \begin{equation}\label{eqn_value_function_stationary}
U_{s_0}^\prime(h) =  \mathbb{E}_{s\sim \rho^\prime, a\sim\pi_h}\left[r(s,a)\right].
\end{equation}
  The advantage of this formulation is that the objective function is now independent of the initial state and therefore estimates of the gradient can be computed without requiring the reinitialization of the trajectory. 

  The convergence to a stationary distribution however, prevent us from achieving cyclical behaviors, for the most part. In particular, a sufficient condition for the convergence is that the Markov chain is aperiodic \cite[Theorem 6.6.4]{durrett2010probability}. Which hints to the fact that in some situations cycles are not achievable under these conditions. Let us consider the following scenario as an example. An agent is required to visit three different locations denoted by states $s_1, s_2$ and $s_3$ and there is a charging station $s_0$. In this scenario is not surprising that the optimal policy is such that it cycles in the different locations and the charging station. Consider that the resulting Markov Chain is such that with probability one we transition from $s_i$ to $s_{i+1}$ for $i=0,\ldots, 2$ and from $s_3$ to $s_0$. In this scenario there exists a stationary distribution that places equal mass in every state, i.e = $\rho(s_i) =1/4$ for all $i=0,\ldots 3$. However, the convergence to this distribution is only guaranteed if the initial distribution is the stationary one. This assumption may not be realistic for this scenario since the agent is  most likely to start in the charging station than in the other locations for instance.

  Even if a stationary distribution is not attainable for the cyclic example just described, the return \eqref{eqn_average_return} is still well defined. Thus,  we could attempt to  extend the theory for continuing tasks starting from \eqref{eqn_average_return}  and   avoiding  \eqref{eqn_value_function_stationary}, without relying on a stationary distribution. However, we argue that the discounted formulation in \eqref{eqn_problem_statement} may be the preferred choice  when transient behaviors are deemed important. Consider for instance the following MDP where the states are defined as $\ccalS = \left\{0, 1,\ldots, 10\right\}$ and the actions are $\ccalA=\left\{-1,1\right\}$. The transition dynamics are such that for all $s\in\ccalS\setminus\left\{0,10\right\}$ we have that $s_{t+1}=s_t + a_t$, in the case of of $s_t = 0$ we have that $s_{t+1} = s_t +a_t\mathbbm{1}(a_t>0)$ and in the case of $s_t=10$ we have $s_{t+1} = s_t$ regardless of the action selected. All the states yield zero rewards except for $s_t = 10$ whose reward is 1. Notice that under any random policy, as long as $P(a_t=1) >0$, the state converges to $10$ and thus the average return \eqref{eqn_average_return} takes the value one. This is the same value as that of choosing the action $a_t = 1$ for any state. The discounted formulation \eqref{eqn_problem_statement} allows us to distinguish between these two policies since the larger the average time to reach the state $s=10$ the smaller the value function. 

Having argue the practical importance of considering problems of the form \eqref{eqn_problem_statement} for continuing tasks we proceed to describe a policy gradient based solution.
}

%
\section{Online Policy Gradient}\label{sec_online}
\subsection{Stochastic Gradient Ascent}\label{sec_offline_alg}
In order to  compute a stochastic approximation of $\nabla_h U_{s_0}(h)$ given in \eqref{eqn_nabla_U} we need to sample from the distribution $\rho_{s_0}(s,a)$ defined in \eqref{eqn_discounted_distribution}. The intuition behind $\rho_{s_0}(s,a)$ is that it weights by $(1-\gamma)\gamma^t$ the probability of the system being at a specific state-action pair $(s,a)$ at time $t$.  Notice that the weight $(1-\gamma)/\gamma^t$ is equal to the probability of a geometric random variable of parameter $\gamma$ to take the value $t$. Thus, one can interpret the distribution $\rho_{s_0}(s,a)$ as the probability  of reaching the state-action pair $(s,a)$ after running the system for $T$ steps, with $T$ randomly drawn from a geometric distribution of parameter $\gamma$, and starting at state $s_0$. The geometric sampling transforms the discounted infinite horizon problem into an undiscounted episodic problem with random horizon (see e.g. \cite[pp.39-40]{bertsekas1996NDP}). This supports steps 2-7 in Algorithm \ref{alg_stochastic_grad} which describes how to obtain a sample $(s_T,a_T)\sim\rho_{s_0}(s,a)$. Then to compute an unbiased estimate of $\nabla_h U_{s_0}(h)$ (cf., Proposition \ref{prop_unbiased_grad}) one can substitute the sample $(s_T,a_T)$  in  the stochastic gradient expression
\begin{equation}\label{eqn_stochastic_gradient}
\hat{\nabla}_h U_{s_0}(h,\cdot) = \frac{1}{1-\gamma}\hat{Q}(s_T,a_T;h)\kappa(s_T,\cdot)\Sigma^{-1}(a_T-h(s_T)),
\end{equation}
with $\hat Q(s_T,a_T;h)$ being an unbiased estimate of $Q(s_T,a_T;h)$.  Algorithm \ref{alg_stochastic_grad} summarizes the steps to compute the stochastic approximation in \eqref{eqn_stochastic_gradient}. We claim that it is unbiased in Proposition \ref{prop_unbiased_grad} as long as the rewards are bounded. We formalize this assumption next as long with some other technical conditions required along the paper.
\begin{assumption}\label{assumption_reward_function}
  There exists $B_r>0$ such that $\forall (s,a) \in \ccalS\times\ccalA$, the reward function $r(s,a)$ satisfies $|r(s,a)|\leq B_r$. In addition $r(s,a)$ has bounded first and second derivatives, with bounds  $|\partial r(s,a)/ \partial s|\leq L_{rs}$ and $|\partial r(s,a)/\partial a |\leq L_{ra}$. 
\end{assumption}
Notice that these assumptions are on the reward which is user defined, as such they do not impose a hard requirement on the problem.
%
\begin{algorithm}
 \caption{StochasticGradient}\label{alg_stochastic_grad}
\begin{algorithmic}[1]
 \renewcommand{\algorithmicrequire}{\textbf{Input:}}
 \renewcommand{\algorithmicensure}{\textbf{Output:}}
 \Require $h$, $s_0$
  \State Draw an integer $T$ form a geometric distribution with parameter $\gamma$, $P(T = t) = (1-\gamma)\gamma^t$
  \State Select action $a_{0} \sim \pi_h(a|{s})$ 
  \For {$t = 0,1,\ldots T-1 $}
  \State Advance system $s_{t+1} \sim P_{s_t\to s_{t+1}}^{a_t}$
  \State Select action $a_{t+1} \sim \pi_h(a|{s_{t+1}})$ 
  \EndFor
  \State Get estimate of $Q(s_T,a_T;h)$ \label{step_Q}
  \State  Compute the stochastic gradient $\hat{\nabla}_hU(h,\cdot)$ as in \eqref{eqn_stochastic_gradient}
 \Return $\hat{\nabla}_hU(h,\cdot)$
 \end{algorithmic}\label{alg_stochastic_gradient} 
\end{algorithm}
%
\begin{proposition}[(Proposition 3 \cite{paternain2018stochastic})] \label{prop_unbiased_grad}
  The output $\hat{\nabla}_h U_{s_0}(h,\cdot)$ of Algorithm \ref{alg_stochastic_gradient} is an unbiased estimate of $\nabla_h U_{s_0}(h,\cdot)$ in \eqref{eqn_nabla_U}.
\end{proposition}

{An unbiased estimate of $Q(s_T,a_T)$ can be computed considering the cumulative reward from $t=T$ until a randomly distributed horizon $T_Q\sim geom(\gamma)$ (cf., Proposition 2 \cite{paternain2018stochastic}). The variance of this estimate may be high resulting on a slow convergence of the policy gradient algorithm (Algorithm \ref{alg_stochastic_grad}). For these reasons, the literature on RL includes several practical improvements. Variance can be reduced by including batch versions of the gradient method, in which several stochastic gradients are averaged before performing the update in \eqref{eqn_stochastic_gradient}. One particular case of a batch gradient iteration in \cite{paternain2018stochastic}, averages two gradients sharing the same state $s_i$ with stochastic actions. Other approaches include the inclusion of baselines \cite{williams1992simple} and actor critic methods \cite{konda2000actor,bhatnagar2009natural,degris2012off}. Irrespective of the form selected to estimate the $Q$ function with the estimate \eqref{eqn_stochastic_gradient} one could update the policy iteratively running stochastic gradient ascent}
\begin{equation}\label{eqn_stochastic_update}
  h_{k+1} = h_k + \eta_k \hat{\nabla}_h U_{s_0}(h_k,\cdot),
\end{equation}
where $\eta_k>0$ is the step size of the algorithm. Under proper conditions stochastic gradient ascent methods can be shown to converge with probability one to the local maxima  \cite{pemantle1990nonconvergence}. This approach has been widely used to solve parametric optimization problems where the decision variables are vectors in $R^n$ and in \cite{paternain2018stochastic} these results are extended to non-parametric problems in RKHSs. 
Observe however, that in order to provide an estimate of $\nabla U_{s_0}(h_k,\cdot)$, Algorithm \ref{alg_stochastic_grad} requires $s_0$ as the initial state. Hence, it is not possible to get estimates of the gradient without resetting the system to the initial state $s_0$, preventing a fully online implementation. {As discussed in Section \ref{sec_continuing},   this is a common challenge in continuing task RL problems and in general the alternative is to modify the objective the function and to assume the existence of a steady-state distribution to which the MDP converges (see e.g., \cite[Chapter 13]{sutton1998reinforcement} or \cite{degris2012off}), to make the problem independent of the initial state. In this work we choose to keep the objective \eqref{eqn_problem_statement} since the ergodicity assumption is not necessarily guaranteed in practice and the alternative formulation makes transient behaviors irrelevant, as it was also discussed in Section \ref{sec_continuing}.  } Notice that, without loss of generality,  Algorithm  \ref{alg_stochastic_grad} can be  initialized at state $s_k$ and its output becomes an unbiased estimate of $\nabla U_{s_k}(h_k,\cdot)$. The main contribution of this work is to show that the gradient of $U_{s_k}(h)$ is also an ascent direction for $U_{s_0}(h)$ (cf., Theorem \ref{prop_all_gradients}) and thus, these estimates can be used to maximize $U_{s_0}(h)$ hence allowing a fully online implementation. We describe the algorithm in the next section.
%

\subsection{Online Implementation}
As suggested in the previous section it is possible to compute unbiased estimates of $\nabla_h U_{s_k}(h_k)$ by running Algorithm \ref{alg_stochastic_gradient} with inputs $h_k$ and $s_k$. The state $s_k$ is defined for all $k\geq 1$ as the state resulting from running the Algorithm \ref{alg_stochastic_gradient} with inputs $h_{k-1}$ and $s_{k-1}$. This is, at each step of the online algorithm \textemdash which we summarize under Algorithm \ref{alg_online_policy_gradient}\textemdash the system starts from state $s_k$ and transits to a state $s_{T_k}$ following steps 3--6 of Algorithm \ref{alg_stochastic_gradient}. Then, it advances from $s_{T_k}$ to $s_{k+1}$ to perform the estimation of the $Q$-function, one that admits an online implementation, for instance by adding  the rewards of the next $T_Q$ steps with $T_Q$ being a geometric random variable. The state $s_{k+1}$ is the initial state for the next iteration of Algorithm \ref{alg_online_policy_gradient}. 
%
\begin{algorithm}
 \caption{Online Stochastic Policy Gradient Ascent}
\begin{algorithmic}[1]
 \renewcommand{\algorithmicrequire}{\textbf{Input:}}
 \renewcommand{\algorithmicensure}{\textbf{Output:}}
 \Require step size $\eta_0$
 \State \textit{Initialize}: $h_0=0$, and draw initial state $s_0$
 \For{$k=0 \ldots$}
 \State Compute the stochastic gradient and next state:\\ $\left(\hat{\nabla}_h U(h_k,\cdot), s_{k+1} \right) = \textrm{StochasticGradient}(h_k, s_k)$
 \State Stochastic gradient ascent step
 $$
\tilde{h}_{k+1} = h_k +\eta_k \hat{\nabla}_h U(h_k,\cdot)
$$
 \State Reduce model order $h_{k+1} =$ KOMP($\tilde{h}_{k+1},\epsilon_K$) 
 \EndFor
 \end{algorithmic}\label{alg_online_policy_gradient}
 \end{algorithm}
Notice that the update \eqref{eqn_stochastic_update} \textemdash step 5 in Algorithm \ref{alg_online_policy_gradient} \textemdash requires the introduction of a new element $\kappa(s_{T_k},\cdot)$ in the kernel dictionary at each iteration, thus resulting in memory explosion. {To overcome this limitation we modify the stochastic gradient ascent by introducing a projection over a RKHS of lower dimension as long as the induced error remains below a given compression budget. This algorithm, which runs once after each gradient iteration,  prunes the kernel expansion that describes the policy $h$ to remove the kernels that are redundant up to an admissible error level $\epsilon_k>0$. The subroutine is known as Kernel Orthogonal Match and Pursuit (KOMP) \cite{vincent2002kernel} \textemdash step 6 in Algorithm \ref{alg_online_policy_gradient}.

  The fundamental reason to do this pruning projection  over a smaller subspace is that it allow us to control the model order of the policy $h_k$, as it is shown in Theorem \ref{theo_convergence}.  However, the induced error translates into a bias on the estimate of $\nabla_h U_{s_k}(h,\cdot)$. We formalize this claim in the next proposition.}
%
\begin{proposition}\label{prop_error_bound}
  The update of Algorithm \ref{alg_online_policy_gradient} is equivalent to running biased stochastic gradient ascent
  \begin{equation}\label{eqn_error_bound}
h_{k+1} = h_k +\eta \hat{\nabla}_h U_{s_k}(h,\cdot)+b_k, 
  \end{equation}
  with bias bounded by the compression budget $\epsilon_K$ for all $k$, i.e., $\left\|b_k\right\|_\ccalH<\epsilon_K$.
\end{proposition}
\begin{proof}
The proof is identical to that in  \cite[Proposition 5]{paternain2018stochastic}.
  \end{proof}
As stated by the previous proposition the effect of introducing the KOMP algorithm is that of updating the policy by running gradient ascent, where now the estimate  is biased. The later will prevent the algorithm to converge to a critical point of the value function. However, we will be able to establish convergence to a neighborhood of the critical point as long as the compression is such that the error introduced is not too large. A difference between the online algorithm and the offline one presented in \cite{paternain2018stochastic} is that even for a compression error of $\epsilon_K=0$ we cannot achieve exact convergence to the critical points, because the directions that are being used to ascend in the function $U_{s_0}(h)$ are in fact estimates of the gradients of $U_{s_k}(h)$. In the next section we discuss this in more detail and we establish 
that the inner product of gradients of $U_{s_k}(h)$ and $U_{s_0}(h)$ is positive when $h$ belongs to a properly selected Gaussian RKHS (Theorem \ref{prop_all_gradients}).


\section{All gradients are ascent directions}\label{sec_important_properties}
As we stated in the previous section, the main difference when comparing the online \textemdash continuing task \textemdash  with the offline setting \cite{paternain2018stochastic} \textemdash episodic task \textemdash is in the gradient of the value function that we estimate. In the online setting, we have access to estimates of the gradient of the value function conditioned on the state $s_k$, that is  $\nabla U_{s_k}(h_k,\cdot)$, where $s_k$ changes from one iteration to another. On the other hand, in  the offline setting we can restart the system to its original state $s_0$ or redraw it from a given distribution $P(s_0)$, so that we can compute estimates of $\nabla U_{s_0}(h_k,\cdot)$ at each iteration. Thus, in the offline case we perform ascent steps over the same function $U_{s_0}$, whereas in the online setting  would perform gradient steps over functions $U_{s_k}$ which are different for each $k$. This main difference is as well a fundamental challenge since in principle we are not guaranteed that the gradients that can be computed are ascent directions of the function of interest $U_{s_0}(h)$. Moreover, a second question is whether  finding the maximum of the value function conditioned at $s_0$ is a problem of interest or not after we reach a new state $s_k$. We answer the second question in Proposition \ref{prop_common_critical_points} by showing that if $h$ is a critical point of $U_{s_k}(h)$ it will also be a critical point of $U_{s_l}(h)$ for all $l\geq k$. The latter can be interpreted in the following way, having a policy that is optimal at a given time, makes it optimal for the future. An in that sense, maximizing the initial objective function is a valid problem, since finding a maximum for that function means that we had found one for all $U_{s_k}(h)$. To formalize this result, we analyze the critical points of $U_{s_k}(h)$. To do so write $\nabla_h U_{s_k}(h,\cdot)$ (cf., \eqref{eqn_nabla_U}) as the following integral
\begin{equation}\label{eqn_gradient_aux}
  \begin{split}
    &\nabla_h U_{s_k} (h,\cdot) = \\ \frac{1}{1-\gamma}&\int Q(s,a;h)k(s,\cdot)\Sigma^{-1}(a-h(s))\rho_{s_k}(s,a) \,ds da,
    \end{split}
  \end{equation}
where $\rho_{s_k}(s,a)$ is the distribution defined in \eqref{eqn_discounted_distribution}. We work next towards writing $\rho_{s_k}$ as a product of a distribution of states and a distribution of actions. To that end, write the MPD transition distribution $p(s_t=s,a_t=a|s_k)$ for any $t\geq k$ as
\begin{equation}
  \begin{split}
    p(s_t=s,a_t=a|s_k) &= p(s_t=s|s_k)\pi_h(a_t=a|s_t,s_k) \\
    &= p(s_t=s|s_k)\pi_h(a_t=a|s_t),
    \end{split}
  \end{equation}
where the last equality follows from the fact that the action depends only on the current state conditional on the policy (cf.,\eqref{eqn_gaussian_policy}). By substituting the previous expression in \eqref{eqn_discounted_distribution}, $\rho_{s_k}(s,a)$ reduces to 
\begin{equation}
\rho_{s_k}(s,a) = (1-\gamma)\sum_{t=k}^\infty \gamma^t \pi_h(a_t=a|s_t=s) p(s_t=s|s_k). 
  \end{equation}
Notice that the density $\pi_h(a_t=a|s_t=s)$ is independent of $t$ and thus, the previous expression yields 
\begin{equation}  
    \rho_{s_k}(s,a) = \pi_h(a|s)(1-\gamma)\sum_{t=k}^\infty \gamma^t  p(s_t=s|s_k).
  \end{equation}
Hence, defining $\rho_{s_k}(s) :=(1-\gamma)\sum_{t=k}^\infty \gamma^t  p(s_t=s|s_k)$, it follows that
$\rho_{s_k}(s,a) = \rho_{s_k}(s) \pi_h(a|s)$. Having $\rho_{s_k}(s,a)$ written as a product of a function depending on the states only and a function depending on the action only, allows us to reduce the expression in \eqref{eqn_gradient_aux} to 
\begin{equation}\label{eqn_new_gradient_expression}
\nabla_h U_{s_k}(h,\cdot) =\frac{1}{1-\gamma} \int D(s)\rho_{s_k}(s)\kappa(s,\cdot)\,ds,
\end{equation}
where the function $D(s)$ is the result of the integration of all the terms that depend on the action
\begin{equation}\label{eqn_derivative_q}
D(s) = \int Q(s,a;h) \Sigma^{-1}\left(a-h(s)\right)\pi_h(a|s)\,da.
\end{equation}
Writing the gradient as in \eqref{eqn_new_gradient_expression}, allows us to split the integrands in the product of a term $\rho_{s_k}(s)$ that depends on the state at time $k$ and a term $D(s)\kappa(s,\cdot)$ that do not depend on $s_k$. Hence, if a policy $h$ is such that $D(s)$ is zero for all $s$, then $h$  is a critical point for all value functions. This idea suggests that the quantity $D(s)$ is of fundamental importance in the problem. Indeed, $D(s)$ is an approximation of the derivative of the $Q$-function with respect to $a$. To see why this is the case, observe that because  $\pi_h(a|s)$ is Gaussian, then $\Sigma^{-1}(a-h(s))\pi_h(a|s)$ is the derivative of $\pi_h(a|s)$ with respect to $a$. Hence, $D(s)$ can be written as
\begin{equation}
D(s) = \int Q(s,a;h) \frac{\partial \pi_h(a|s)}{\partial a} da.
  \end{equation}
Notice that as the covariance matrix of $\pi_h(a|s)$  approaches the null matrix, the distribution $\pi_h(a|s)$ approaches a Dirac delta centered at $h(s)$. That being the case, $D(s)$ yields
\begin{equation}
D(s) \simeq \int Q(s,a;h) \delta^{\prime}(a-h(s)) \, da = \frac{\partial Q(s,a;h)}{\partial a}\Big|_{a=h(s)}.
  \end{equation}
In this case, the fact that $D(s)$ is identically zero means that we have found a policy $h$ that makes every action a stationary point of the $Q$-function. The previous observation relates to Bellman's optimality condition, that establishes that a policy is optimal if it is such that it selects the actions that maximize the $Q$-function. When $\Sigma$ is different than the null matrix, $D(s)$ is an approximation of the derivative. In \cite{nesterov2017random} the aforementioned Gaussian smoothing is used as an approximation of the stochastic gradient in the context of zero-order optimization, and is formally established that it approximates the derivative with an error that depends linearly on the norm of the covariance matrix. The insights provided in the previous paragraphs regarding the importance of the function $D(s)$ are not enough to fully characterize the critical points of $U_{s_k}(h)$ since according to \eqref{eqn_new_gradient_expression} its gradient depends as well on the long run discounted distribution $\rho_{s_k}(s)$. The fact that this distribution might take the value zero at different states for different $s_k$ does not allow us to say that a policy can be a critical point of every value function. However, we will be able to prove that if a policy is a critical point for $U_{s_k}(h)$ it is also a critical point for $U_{s_l}(h)$ for every $l\geq k$. To formalize the previous statement we require the following auxiliary result.

\begin{lemma}\label{lemma_distribution}
  Let $S_0$ be the following set
\begin{equation}
  \ccalS_0 = \left\{ s\in \ccalS : \exists \,t \geq 0,  p(s_t=s|s_0) >0\right\}.
\end{equation}

  For all $s^\prime,s\in \ccalS_0$ and $s^{\prime\prime}\in \ccalS \setminus \ccalS_0$ we have that
  \begin{equation}
\rho_{s_0}(s)>0 \quad \mbox{and} \quad   \rho_{s^\prime}(s^{\prime\prime}) = 0.
    \end{equation}
  \end{lemma}
\begin{proof}
  See appendix \ref{appendix_lemma_distribution}.
  \end{proof}
The set $\ccalS_0$ contains the states for which the probability measure conditioned on the policy and on the initial state $s_0$ is strictly positive. We term $S_0$ the set of reachable states from $s_0$. The previous result, ensures that for all reachable states $s\in\ccalS_0$, the probability measure $\rho_{s_0}(s)$ is strictly positive. The latter is not surprising, since intuitively, the distribution $\rho_{s_0}(s)$ is a weighted sum of the distributions of reaching the state $s$ starting from $s_0$ at different times. Moreover, and along the same lines, we establish that if a point cannot be reached starting from $s_0$, it cannot be reached starting from any other point that is reachable from $s_0$. The previous result can be summarized by saying that set of reachable points does not increase as the system evolves, i.e., $\ccalS_k \subseteq \ccalS_0$ for all $k\geq 0$. Building on the previous result we show that the set of critical points of $U_{s_k}(h)$ can only increase with the iterations. This means that a critical point of the functional $U_{s_k}(h)$ is also a critical point of the functional conditioned at any state visited in the future. Without loss of generality we state the result for $k=0$ { and with the dimension of the action space $p=1$} .
%
\begin{proposition}\label{prop_common_critical_points}
  If $h\in\ccalH$ is a critical point of $U_{s_0}(h)$, then it is also a critical point for $U_{s_{l}}(h)$ for all $l \geq 0$, with $s_l\in  \ccalS_0$.
\end{proposition}
\begin{proof}
  Let us start by writing the square of norm of $\nabla_h U_{s_0}(\cdot)$ according to \eqref{eqn_new_gradient_expression} as
  \begin{equation}\label{eqn_gradient_norm}
    \left\|\nabla_h U_{s_0}(h,\cdot) \right\|^2 = 
    \int \int D(s)\rho_{s_0}(s)\kappa(s,s^\prime) D(s^\prime)\rho_{s_0}(s^\prime) \, dsds^\prime.
  \end{equation}
  From Mercer's Theorem (cf., \cite{j1909xvi}) there exists $\lambda_i > 0$ and orthornormal basis $e_i(s)$ of $L^2(\ccalS)$ such that
  \begin{equation}\label{eqn_mercer}
    \kappa(s,s^\prime) = \sum_{i=1}^\infty \lambda_i e_i(s) e_i(s^\prime).
  \end{equation}
  Using the previous result, we can decompose the expression in \eqref{eqn_gradient_norm} as the following sum of squares
  \begin{equation}\label{eqn_norm_of_gradient}
    \begin{split}
    \left\|\nabla_h U_{s_0}(h,\cdot) \right\|^2 =
    \sum_{i=1}^\infty \lambda_i\left[\int D(s)\rho_{s_0}(s)e_i(s) \, ds\right]^2.
    \end{split}
  \end{equation}
  Notice that the previous expression can take the value zero if and only if for all $i=1.\ldots$ we have that
  \begin{equation}\label{eqn_crit_point_condition}
\int D(s)\rho_{s_0}(s)e_i(s)ds =0.
    \end{equation}
  Because $e_i(s)$ with $i=1\ldots $ form an orthogonal basis of $L^2(\ccalS)$ it means that \eqref{eqn_crit_point_condition} holds if and only if $D(s)\rho_{s_0}(s) \equiv 0$. To complete the proof, we are left to show that if $D(s)\rho_{s_0}(s) \equiv 0$ then it holds that $D(s)\rho_{s_l}(s) \equiv 0$ for all $l\geq 0$. The latter can be established by showing that for any $s \in S$ such that $\rho_{s_0}(s) =0$ we also have that $\rho_{s_l}(s) = 0$ which follows by virtue of Lemma \ref{lemma_distribution}.
  \end{proof}
The previous result formalizes the idea that if we find an optimal policy at given time, then it is optimal for all future states. Moreover, the latter is true for every critical point, which suggests, that the value functions conditioned at different initial states should be similar. We formalize this intuition in Theorem \ref{prop_all_gradients} where we show that $\nabla_h U_{s_k}(h)$ is an ascent direction for $U_{s_0}(h)$ if the distribution $\rho_{s_k}(s)$ is bounded above and bounded away from zero. We also require some smoothness assumptions on the transition probability which we formalize next. 
%
\begin{assumption}\label{assumption_prob_dist}
  There exists $\beta_{\rho}>0$ and $B_{\rho}$ such that for all $s_k,s\in\ccalS_0$ and for all $h\in\ccalH$ we have that
  \begin{equation}
B_{\rho}    \geq \rho_{s_k}(s) \geq \beta_{\rho}.
  \end{equation}
 In addition we have that the transition probability is Lipschitz with constant $L_p$, i.e.,
    \begin{equation}
     \left|p(s_t=s|s_{t-1},a_{t-1})-p(s_t=s^\prime|s_{t-1},a_{t-1}) \right| \leq L_p\left\|s-s^\prime\right\|.    \end{equation}
    We require as well the following smoothness properties of the probability transition
    \begin{equation}
      p^\prime(s,a):= \frac{\partial p(s_{t+1}|s_t,a_t)}{\partial a_t}\big|_{s_t=s,a_t=a}
    \end{equation}
    to be Lipschitz with constants $L_{ps}$ and $L_{pa}$, this is
    \begin{equation}
      \left|p^\prime(s,a)-p^\prime(s^\prime,a^\prime)\right| \leq L_{ps}\left\|s-s^\prime\right\|+L_{pa}\left\|a-a^\prime\right\|. 
    \end{equation}
\end{assumption}
Notice that the lower bound on $\beta_\rho(s)$ requires that every state is reachable. The latter can be achieved with any sufficiently exploratory policy unless there are states that are attractive. The previous assumptions allow us to establish that $\nabla_h U_{s_k}(h)$ is an ascent direction for the function $U_{s_0}(h)$. Notice that for the latter to hold, we require that
\begin{equation}
\left\langle\nabla_hU_{s_0}(h), \nabla_hU_{s_k}(h)\right\rangle_{\ccalH} \geq 0.
  \end{equation}
By writing the gradient as in \eqref{eqn_new_gradient_expression} and using the reproducing property of the kernel it follows that the previous condition is equivalent to 
\begin{equation}
 \int D(s)^\top\rho_{s_0}(s)\kappa(s,s^\prime)D(s^\prime)\rho_{s_k}(s^\prime) \, ds ds^\prime\geq 0.
\end{equation}
where $(\cdot)^\top$ denotes transpose. Notice that if $\kappa(s,s^\prime)$ approaches a Dirac delta, the integral with respect to $s^\prime$, in the limit reduces to evaluating $D(s^\prime)\rho_{s_k}(s^\prime)$ at $s^\prime=s$. Thus, the double integral is an approximation of
\begin{equation}
\int \|D(s)\|^2\rho_{s_0}(s)\rho_{s_k}(s) \, ds
\end{equation}
which is always non-negative. To formalize the previous argument we will consider a Gaussian Kernel and we will show that if the width of the kernel is small enough, then the previous result holds (Theorem \ref{prop_all_gradients}). We require to establish first that $D(s)\rho_{s_0}(s)$ is bounded and Lipschitz. This is subject of the following lemma.  
%
\begin{lemma}\label{lemma_lipshitz}
  Let $\kappa_{\sigmah}(s,s^\prime)$ be a matrix-valued Gaussian kernel with covariance matrix  $\Sigma_{\ccalH}\succ 0$, i.e. for all $i=1,\ldots, p$ we have that 
  \begin{equation}\label{Gaussian_lemma_lipschitz}
    \kappa_{\sigmah}(s,s^\prime)_{ii} = e^{-(s-s^\prime)^\top \Sigma_{\ccalH}^{-1}(s-s^\prime)/2}, 
  \end{equation}
  and $\kappa(s,s^\prime;\sigmah)_{ij} = 0$ for all $j=1\ldots p$ with $j\neq i$. Let $B_r, L_{rs}$ and $L_{ra}$ be the constants defined in Assumption \ref{assumption_reward_function}. Likewise, let $B_{\rho}, L_p, L_{ps}$ and $L_{pa}$ be the constants defined in Assumption \ref{assumption_prob_dist}. Furthermore, define the following constants 
  \begin{equation}\label{eqn_bound_D}
    B_D:=\frac{\sqrt{2}B_r}{1-\gamma}\frac{\Gamma\left(\frac{p+1}{2}\right)}{\Gamma\left(\frac{p}{2}\right)},
    \end{equation}
    with $\Gamma(\cdot)$ being the Gamma function, 
    $L_{Qs}=L_{rs}+\frac{B_r}{1-\gamma}L_{ps} |\ccalS|$, $L_{Qa}=L_{ra}+\frac{B_r}{1-\gamma}L_{pa} |\ccalS|$, %
    \begin{equation}
      L_h :=\left\|h\right\|\lambda_{\min}(\Sigma_\ccalH)^{-1/2},
    \end{equation}
    and $L_{D}:=L_{Qs}+L_{Qa}L_h$. Then, we have that $D(s)\rho_{s_k}(s)$ for any $s_k\in \ccalS_0$ is bounded by $B:=B_\rho B_D$ and it is Lipschitz with constant $L:=B_DL_p+B_\rho L_D$. 
\end{lemma}
\begin{proof}
See Appendix \ref{appendix_lipshitz}.
  \end{proof}
%
As it was previously discussed we require a Gaussian Kernel whose width is small enough for the inner product of gradients at different initial states to be positive. We next formalize this condition. Define the normalization factor $Z:= \sqrt{\det{2\pi \Sigma_\ccalH}}$ and let
\begin{equation}\label{eqn_kernel_condition}
\sqrt{np}\left(1+\frac{\beta_\rho}{B_\rho}\right)\left\|\Sigma_\ccalH\right\|ZL(h,\Sigma_{\ccalH}) B|\ccalS|\leq \frac{\varepsilon}{2}\frac{\beta_\rho}{B_\rho}.
  \end{equation}
The previous condition in a sense defines the maximum width of the kernel. Since if the norm of $\Sigma_\ccalH$ is large, the previous condition cannot hold. This intuition is not exact since the term $Z$ includes the determinant of the matrix $\Sigma_\ccalH$ and thus, it is possible to have a kernel that has some directions being wide as long as the product of the eigenvalues is small enough. Likewise the Lipschitz constant in \eqref{eqn_kernel_condition} depends on the norm of the function $h$, and in that sense it is necessary to ensure that the norm remains bounded for said condition to hold. We are now in conditions of establishing the  main result in this work, which states that as long as the norm of $\nabla_h U_{s_0}(h)$ is large, the gradient of any value function $\nabla_h U_{s_k}(h)$ is an ascent direction for $U_{s_0}(h)$. This result will be instrumental also to the proof of convergence of the online algorithm (Section \ref{sec_convergence}). 
%
%
\begin{theorem}\label{prop_all_gradients}
  Under the hypotheses of Lemma \ref{lemma_lipshitz}, for every $\varepsilon>0$ and for every $\ccalH$ and $h\in\ccalH$ satisfying \eqref{eqn_kernel_condition} it holds that if $\left\| \nabla_h U_{s_0}(h,\cdot)\right\|^2_{\ccalH} \geq \varepsilon$ then  we have that for all $k\geq 0$ 
  \begin{equation}\label{eqn_prod_diff_grad}
    \left\langle\nabla_hU_{s_0}(h,\cdot),\nabla_hU_{s_k}(h,\cdot) \right\rangle_{\ccalH}>{\frac{\varepsilon}{2}\frac{\beta_\rho}{B_{\rho}}}.
  \end{equation}
  \end{theorem}
\begin{proof}
  %
  %
  Consider the following integral, with the kernel covariance matrix $\sigmah$ as a parameter
  \begin{equation}\label{eqn_I}
    I_{\sigmah} = \int D(s)^\top\rho_{s_l}(s)\kappa_{\sigmah}(s,s^\prime)\rho_{s_k}(s^\prime)D(s^\prime)\,dsds^\prime, 
  \end{equation}
  where $\kappa_{\sigmah}(s,s^\prime)$ is a kernel of the form \eqref{Gaussian_lemma_lipschitz}. Observe that by writing the gradients of $U_{s_0}(h)$ and $U_{s_k}(h)$ as in \eqref{eqn_new_gradient_expression}, it follows that  $I_{\Sigma_{\ccalH}}$  is the inner product in \eqref{eqn_prod_diff_grad}. Hence, to prove the claim, it suffices to show that for all $\sigmah$ satisfying condition \eqref{eqn_kernel_condition},  $I_{\sigmah}> {\varepsilon \beta_\rho/(2B_\rho)}$. To do so, apply the change of variables $u=s^\prime-s$, and divide and multiply the previous expression by $Z:=\sqrt{\det{2\pi \sigmah}}$ to write $I_{\sigmah}$ as
\begin{align}\label{eqn_def_ivarsigma}
\nonumber I_{\sigmah}& =\int D(s)^\top\rho_{s_0}(s)\kappa_{\sigmah}(s,s+u)\rho_{s_k}(s+u) D(s+u)  ds du\\
&\hspace{-0.4cm}=Z \int D(s)^\top\rho_{s_0}(s)g(u;0,\sigmah)\rho_{s_k}(s+u) D(s+u) \, ds du. 
  \end{align}
where the normalization factor $Z$ was introduced to identify  $g(u;0,\sigmah):=\kappa_{\sigmah}(s,s+u)_{ii}/Z$ as a  Gaussian probability density function with zero mean and covariance $\sigmah$ (cf. \eqref{Gaussian_lemma_lipschitz}).  
Then we write the partial integral with respect to $u$ as the expectation of $D(s+u)\rho_{s_k}(s+u)$,
\begin{equation}
I_{\sigmah}\hspace{-2pt}=Z\hspace{-2pt}\int\hspace{-2pt} D(s)^\top\rho_{s_0}(s)\mathbb{E}_{u\sim\ccalN(0,\sigmah)}\left[D(s+u) \rho_{s_k}(s+u)\right] \, ds.
  \end{equation}
From Lemma \ref{lemma_lipshitz} it follows that $D(s)\rho_{s_k}(s)$ is Lipschitz with constant $L$. Then, by virtue of \cite[Theorem 1]{nesterov2017random} we have that 
\begin{equation}
\left\| \mathbb{E}\left[D(s+u) \rho_{s_k}(s+u)\right] - D(s)\rho_{s_k}(s)\right\| \leq \sqrt{np}\left\|\sigmah\right\| L.\label{eq:nesetrov}
\end{equation}
where again the expectation is taken with respect to the random variable $u\sim\ccalN(0,\sigmah)$.
The result in \eqref{eq:nesetrov}  allows us to lower bound $I_{\sigmah}$ by 
\begin{align}
    I_{\sigmah} &\geq  {Z}\int \|D(s)\|^2\rho_{s_0}(s)\rho_{s_k}(s) \, ds \nonumber \\
    &-{Z}\sqrt{np}\left\|\sigmah\right\| L\int \left\|D(s)\rho_{s_0}(s)\right\| \, ds \nonumber\\
    &=\bar I_{\sigmah}-\sqrt{np}\left\|\sigmah\right\| {Z}L\int \left\|D(s)\rho_{s_0}(s)\right\| \, ds \label{eqn_diff_I}
     \end{align}
where $\bar I_{\sigmah}$ was implicitly defined in \eqref{eqn_diff_I} as  
\begin{align}\bar I_{\sigmah}:={Z}\int \left\|D(s)\right\|\rho_{s_0}(s)\rho_{s_k}(s)ds\\=\sqrt{\det{2\pi \sigmah}} \int \left\|D(s)\right\|^2\rho_{s_0}(s)\rho_{s_k}(s)ds \end{align} 
Let us next define the following integrals, identical to  $I_{\sigmah}$ and $\bar I_{\sigmah}$, but for  $\rho_{s_0}$ substituting $\rho_{s_k}$ 
%
\begin{align}
J_{\sigmah}&:= \int D(s)^\prime\rho_{s_0}(s)\kappa_{\sigmah}(s,s^\prime) D(s^\prime)\rho_{s_0}(s^\prime) \, dsds^\prime\\
\bar J_{\sigmah}&:= Z \int \|D(s)\|^2\rho^2_{s_0}(s)  ds.
\end{align}    
and use the bounds on the probability distribution (cf., Assumption \ref{assumption_prob_dist}) to write
  \begin{equation}
    \begin{split}
\bar J_{\sigmah}&\leq  {Z} \int \|D(s)\|^2\rho_{s_0}(s)\frac{B_p}{\beta_\rho}\rho_{s_k}(s) \, ds = \frac{B_p}{\beta_\rho}\bar I_{\sigmah}.
      \end{split}
  \end{equation}
  Hence, we can write \eqref{eqn_diff_I} as
  \begin{equation}\label{eqn_Ivarsigma}
I_{\sigmah} \geq \frac{\beta_\rho}{B_{\rho}}\bar J_{\sigmah} - \sqrt{np}\left\|\sigmah\right\| ZL\int \left\|D(s)\rho_{s_0}(s)\right\| \, ds.
  \end{equation}
  Repeating steps \eqref{eqn_def_ivarsigma}-\eqref{eqn_diff_I}, after substituting $\rho_{s_0}$ for $\rho_{s_k}$, we can bound the difference between   $J_{\sigmah}$ and $\bar J_{\sigmah}$ as we did it  for $I_{\sigmah}$ and $\bar I_{\sigmah}$ in  \eqref{eqn_diff_I}. Specifically, the following inequality holds 
\begin{align}
    \bar J_{\sigmah} &\geq   J_{\sigmah}-\sqrt{np}\left\|\sigmah\right\| ZL\int \left\|D(s)\rho_{s_0}(s)\right\| \, ds \label{eqn_diff_J}
     \end{align}
 This allows us to further lower bound $I_{\sigmah}$ by
\begin{equation}
I_{\sigmah} \geq \frac{\beta_\rho}{B_{\rho}}J_{\sigmah} - \sqrt{np}\left(1+\frac{\beta_\rho}{B_\rho}\right)\left\|\sigmah\right\| ZL\int \left\|D(s)\rho_{s_0}(s)\right\| \, ds.
  \end{equation}
By virtue of Lemma \ref{lemma_lipshitz} we have that $\left\|D(s)\rho_{s_0}(s)\right\|\leq B$. Defining $|\ccalS|$ as the measure of the set $\ccalS$, $I_{\sigmah}$ can be further lower bounded by
\begin{equation}
I_{\sigmah} \geq \frac{\beta_\rho}{B_{\rho}}J_{\sigmah} - \sqrt{np}{\left(1+\frac{\beta_\rho}{B_\rho}\right)}\left\|\sigmah\right\| ZL B|\ccalS|.
  \end{equation}
Notice that $J_{\sigmah}=\left\| \nabla_h U_{s_0}(h,\cdot)\right\|^2 \geq \varepsilon$, hence the previous inequality reduces to 
\begin{equation}
I_{\sigmah} \geq \frac{\beta_\rho}{B_{\rho}}{\varepsilon} -\sqrt{np} \left(1+\frac{\beta_\rho}{B_\rho}\right)\left\|\sigmah\right\| ZL B|\ccalS|.
  \end{equation}
Then, for any $\sigmah$ satisfying \eqref{eqn_kernel_condition}, we can lower bound the right hand side of the previous expression by $\varepsilon\beta_\rho/(2B_{\rho})$, obtaining
        \begin{equation}
I_{\sigmah} \geq  {\frac{\beta_\rho}{2B_{\rho}}\varepsilon,}
        \end{equation}
        which completes the proof of the theorem.
\end{proof}

The previous result establishes that for kernels that satisfy the condition \eqref{eqn_kernel_condition} with $h$ outside of an $\varepsilon$ neighborhood of the critical points, i.e., for $h$ such that $\left\|\nabla_h U_{s_0}(h)\right\| > \varepsilon$,  the inner product between $\nabla_h U_{s_k}(h)$ and $\nabla_h U_{s_0}(h)$ is larger than a constant that depends on $\varepsilon$. The latter means that for all state $s_k\in\ccalS$, $\nabla_h U_{s_k}(h)$ is an ascent direction of the function $U_{s_0}(h)$. In the next section we exploit this idea to show that the online gradient ascent algorithm proposed in Section \ref{sec_online} converges with probability one to a neighborhood of a critical point of $U_{s_0}(h)$. 
%
  %
  %
  %
    %
    %

%

\section{Convergence Analysis of Online Policy Gradient}\label{sec_convergence}
Let $\left(\Omega,\ccalF,P\right)$ be a probability space and define the following sequence of increasing sigma-algebras $\left\{\emptyset, \Omega \right\} = \ccalF_0 \subset\ccalF_1 \subset \ldots \subset \ccalF_k \subset \ldots \subset \ccalF_{\infty}\subset \ccalF$, where for each $k$ we have that $\ccalF_k$ is the sigma algebra generated by the random variables $h_0,\ldots,h_k$. For the purpose of constructing a submartingale that will be used in the proof of convergence, we provide   a lower bound on the expectation of random variables $U_{s_0}(h_{k+1})$ conditioned to the sigma field $\ccalF_k$ in the next Lemma
\begin{lemma}\label{lemma_supermartingale_pars}
  Choosing the compression budget $\epsilon_{K} = K\eta$ with $K>0$, the sequence of random variables $U(h_k)$ satisfies the following inequality 
\begin{equation}\label{eqn_supermartingale_equation_pars}
    \begin{split}
      \mathbb{E}\left[{U_{s_0}(h_{k+1})|\ccalF_k}\right] \geq  U_{s_0}(h_k)- {\frac{\eta^2}{Z}} C_1-{\frac{\eta^3}{Z^{3/2}}}C_2\\
      -\left\|\nabla_h U_{s_0}(h_k)\right\|_\ccalH K\eta +\eta\left<\nabla_h U_{s_0}(h_k),\nabla_h U_{s_k}(h_k)\right>_\ccalH, 
\end{split}
\end{equation}
where $C_1$ and $C_2$ are the following positive constants
\begin{equation}\label{eqn_bound_for_radius_convergence_1}
  C_1 = L_1\left(\sigma^2+2K\sigma +K^2\right)
\end{equation}
and
\begin{equation}\label{eqn_bound_for_radius_convergence_2}
C_2  =L_2\left(\sigma^2+2K\sigma +K^2\right)^{3/2},
\end{equation}
where  $L_1$ and $L_2$ are given by
\begin{equation}\label{eqn_grad_lipschitz}
L_1 = B_r\frac{(1-\gamma+p(1+\gamma))}{\lambda_{\min}{\Sigma}(1-\gamma)^2}, L_2 = B_r\frac{(1+p)\sqrt{p}}{\lambda_{\min}(\Sigma)^{3/2}(1-\gamma)^3}, 
\end{equation}
and
\begin{equation}\label{eqn_sigma_moment}
\sigma = \frac{(3\gamma)^{1/3}}{\lambda_{\min}(\Sigma^{1/2})(1-\gamma)^2}\left(4\frac{\Gamma(2+p/2)}{\Gamma(p/2)}\right)^{1/4}.
  \end{equation}
  \end{lemma}
\begin{proof}
  Start by writing the Taylor expansion of $U_{s_0}(h_{k+1})$ around $h_k$
  \begin{equation}
    \begin{split}
       U_{s_0}(h_{k+1})= U_{s_0}(h_k)+\left<\nabla_h U_{s_0}(f_k,\cdot),h_{k+1}-h_k\right>_{\ccalH}.
      \end{split}
  \end{equation}
  where $f_k = \lambda h_k +(1-\lambda)h_{k+1}$ with $\lambda\in[0,1]$. From \cite[Lemma 5]{paternain2018stochastic} we have that
  \begin{equation}\label{eqn_lipschitz_nabla_U}
\left\| \nabla_h U_{s_0}(g)-\nabla_hU_{s_0}(h)\right\|_{\ccalH} \leq L_1\left\|g-h\right\|_{\ccalH}+L_2\left\|g-h\right\|_{\ccalH}^2,
    \end{equation}
  with $L_1$ and $L_2$ being the constants in \eqref{eqn_grad_lipschitz}. Adding and subtracting $\left<\nabla_h U_{s_0}(h_k,\cdot),h_{k+1}-h_k\right>_{\ccalH}$ to the previous expression, using the Cauchy-Schwartz inequality and \eqref{eqn_lipschitz_nabla_U} we can re write the previous expression as 
  \begin{equation}
    \begin{split}
     U_{s_0}(h_{k+1}) 
     &= U_{s_0}(h_k)+\left<\nabla_h U_{s_0}(h_k,\cdot),h_{k+1}-h_k\right>_{\ccalH} \\
     &+\left<\nabla_h U_{s_0}(f_k,\cdot)-\nabla_hU_{s_0}(h_k,\cdot),h_{k+1}-h_k\right>_{\ccalH}\\
     & \geq U_{s_0}(h_k)+\left<\nabla_h U_{s_0}(h_k,\cdot),h_{k+1}-h_k\right>_{\ccalH} \\
     &-L_1\left\|h_{k+1}-h_k\right\|_{\ccalH}^2-L_2\left\|h_{k+1}-h_k\right\|_{\ccalH}^3.
      \end{split}
  \end{equation}
  Let us consider next the conditional expectation of the random variable $U_{s_0}(h_{k+1})$ with respect to the sigma-field $\ccalF_k$. Combine the monotonicity and the linearity of the expectation with the fact that $h_k$ is measurable with respect to $\ccalF_k$ to write
  \begin{equation}\label{eqn_first_supermartingale_bound_pars}
    \begin{split}
      \mathbb{E}\left[{U_{s_0}(h_{k+1})|\ccalF_k}\right] \geq U_{s_0}(h_k)\\
      +\left<\nabla_h U_{s_0}(h_k,\cdot),\mathbb{E}\left[h_{k+1}-h_k|\ccalF_k\right]\right>_{\ccalH} \\
      -L_1\mathbb{E}\left[\left\|h_{k+1}-h_k\right\|_{\ccalH}^2|\ccalF_k\right]-L_2\mathbb{E}\left[\left\|h_{k+1}-h_k\right\|_{\ccalH}^3|\ccalF_k\right].
      \end{split}
    \end{equation}
  Using the result of Proposition \ref{prop_error_bound} we can write the expectation of the quadratic term in the right hand side of \eqref{eqn_first_supermartingale_bound_pars} as 
  \begin{equation}
    \begin{split}
      L_1\mathbb{E}\left[\left\|h_{k+1}-h_k\right\|_{\ccalH}^2|\ccalF_k\right] 
      \\
      \leq L_1\eta^2 \mathbb{E}\left[\left\|\hat{\nabla}_h U_{s_k}(h_k,\cdot)\right\|_{\ccalH}^2|\ccalF_k\right]+L_1\epsilon_{K}^2 \\
      +2L_1\eta\epsilon_{K}\mathbb{E}\left[\left\|\hat{\nabla}_h U_{s_k}(h_k,\cdot)\right\|_{\ccalH}|\ccalF_k\right] ,
      \end{split}
    \end{equation}
  Using the bounds on the moments of the estimate (cf., \cite[Lemma 6]{paternain2018stochastic}), the previous expression can be upper bounded by
  \begin{equation}
    \begin{split}
      L_1\mathbb{E}\left[\left\|h_{k+1}-h_k\right\|_{\ccalH}^2|\ccalF_k\right] &\leq \eta^2L_1\left(\sigma^2+2\frac{\epsilon_{K}}{\eta}\sigma +\frac{\epsilon_{K}^2}{\eta^2}\right).
      \end{split}
    \end{equation}
  where $\sigma$ is the constant in \eqref{eqn_sigma_moment}. Choosing the compression budget as $\epsilon_{K} = K \eta$ and using the definition of $C_1$ in \eqref{eqn_bound_for_radius_convergence_1} it follows that   
  \begin{equation}
      L_1\mathbb{E}\left[\left\|h_{k+1}-h_k\right\|_{\ccalH}^2|\ccalF_k\right]     = \eta^2L_1\left(\sigma^2+2K\sigma +K^2\right)=\eta^2 C_1.
  \end{equation}
  Likewise, we have that
  \begin{equation}
    \begin{split}
      L_2\mathbb{E}\left[\left\|h_{k+1}-h_k\right\|_{\ccalH}^3|\ccalF_k\right] &\leq \eta^3 L_2\left(\sigma^2+2\frac{\epsilon_{K}}{\eta}\sigma +\frac{\epsilon_{K}^2}{\eta^2}\right)^{3/2}\\
      &=      \eta^3 L_2\left(\sigma^2+2K\sigma +K^2\right)^{3/2} \\
      &= \eta^3C_2.
      \end{split}
    \end{equation}
Replacing the previous two bounds regarding the moments of  $\left\|h_{k+1}-h_k\right\|$ in \eqref{eqn_first_supermartingale_bound_pars} reduces to  
  \begin{equation}
    \begin{split}
      \mathbb{E}\left[{U(h_{k+1})|\ccalF_k}\right] &\geq U(h_k)-\eta^2C_1-\eta^3C_2\\
&      +\left<\nabla_h U(h_k),\mathbb{E}\left[h_{k+1}-h_k|\ccalF_k\right]\right>_{\ccalH}. 
\end{split}
    \end{equation}
  Using the result of Proposition \ref{prop_error_bound} and the fact that $\hat{\nabla}_hU_{s_k}(h_k)$ is unbiased (cf., Proposition \ref{prop_unbiased_grad}) we can write the inner product in the previous equation as
  \begin{equation}
    \begin{split}
      \left<\nabla_h U_{s_0}(h_k),\mathbb{E}\left[h_{k+1}-h_k|\ccalF_k\right]\right>_{\ccalH} =\\
       \eta\left<\nabla_h U_{s_0}(h_k),\nabla_hU_{s_k}(h_k)\right>_{\ccalH}
      +\left<\nabla_h U_{s_0}(h_k), b_k\right>_{\ccalH}.
      \end{split}
    \end{equation}
The proof is then completed using the Cauchy-Schwartz inequality and fact that the norm of the bias is bounded by $\epsilon_{K} = K\eta$ (cf., Proposition \ref{prop_error_bound}).
  \end{proof}
The previous lemma establishes a lower bound on the expectation of $U_{s_0}(h_{k+1})$ conditioned to the sigma algebra $\ccalF_k$. This lower bound however, is not enough for $U_{s_0}(h_k)$ to be a submartingale, since the sign of the term added to $U_{s_0}(h_k)$ in the right hand side of \eqref{eqn_supermartingale_equation_pars} is not necessarily positive. The origin of this is threefold. The first two reasons stem from algorithmic reasons. These are that we are using the estimate of $\nabla_hU_{s_k}(h_k)$ to ascend on the functionial $U_{s_0}(h)$ -- which does not guarantee the inner product to be always positive -- the bias that results from projecting into a lower dimension via the KOMP algorithm as stated in Proposition \ref{prop_error_bound}. The third reason comes from the analysis in Lemma \ref{lemma_supermartingale_pars} where we bounded the value of the functional using a first order approximation. To overcome the first limitation we will use the result from Theorem \ref{prop_all_gradients} that guarantees that the inner product in the right hand side of \eqref{eqn_supermartingale_equation_pars} is lower bounded by $\varepsilon\beta_{\rho}/(2B_\rho)$ as long as $\left\|\nabla_h U_{s_0}(h)\right\|^2>\varepsilon$. The latter suggests that the definition of the following stopping time is necessary for the analysis 
  \begin{equation}\label{eqn_stopping_time}
N = \min_{k\geq 0} \left\{\left\|\nabla_h U_{s_0}(h_{k},\cdot)\right\|_{\ccalH}^2 \leq \varepsilon  \right\}.
    \end{equation}
  We will show that by choosing the compression factor $\epsilon_{K}$ and the step size sufficiently small we can overcome the other two limitations and establish that $U_{s_0}(h_k)$ is a submartingale as long as $k<N$. To be able to use the result of Theorem \ref{prop_all_gradients} we require, condition \eqref{eqn_kernel_condition} to be satisfied. As previously explained, this requires the norm of $h$ not to grow unbounded, yet due to the stochastic nature of the update there is no guarantees that this will be the case. We assume, however, that policies with infinite norm are poor policies which leads to the conclusion that if the norm of the gradient is not too small, then it has to be the case that the norm of $h$ is bounded. We formalize these ideas next. 
{
\begin{assumption}\label{assumption_large_h}
For every $\ccalH$ it follows that $\lim_{\left\|h\right\|\to \infty} U_{s_0}(h) = \min_{h\in\ccalH} U_{s_0}(h)$.
  \end{assumption}
}
%
{
\begin{lemma}\label{lemma_bound_h}
For every $\varepsilon>0$ there exists a constant $B_h(\varepsilon)$ such that if $\left\|\nabla_h U_{s_0}(h)\right\|^2 \geq \varepsilon$ then $ \left\| h\right\| \leq  B_h(\varepsilon)$. 
\end{lemma}
\begin{proof}
  Since the function $U_{s_0}(h)$ is bounded (cf., \cite[Lemma 1]{paternain2018stochastic}), Assumption \ref{assumption_large_h} implies that $\lim_{\left\|h\right\|\to \infty} \nabla_hU_{s_0}(h) = 0$ and therefore for every $\varepsilon>0$ there exists $B_h(\varepsilon)>0$ such that if $\left\|h\right\| > B_h(\varepsilon)$ then $\left\|\nabla_hU_{s_0}(h)\right\| <\varepsilon$. Hence it has to be the case that if $\left\|\nabla_hU_{s_0}(h)\right\|^2 \geq \varepsilon$, then $\left\|h\right\| \leq B_h(\varepsilon)$. 
\end{proof}
}
%
  As previously discussed to guarantee that $U_{s_0}(h_k)$ is a submartingale we need to choose the compression budget $\epsilon_{K}$ and the step-size $\eta$ small enough. In particular observe that the compression budget multiplies a term that depends on the norm of the gradient in \eqref{eqn_supermartingale_equation_pars}. Hence, to be able to guarantee that reducing the compression budget is enough to have a submartingale we require that the norm of the gradient of the value function is bounded. This is the subject of the following lemma. 
  %
  \begin{lemma}\label{lemma_bound_nabla_U}
    The norm of the gradient of $\nabla_h U_{s_0}(h)$ is bounded by $B_{\nabla}$ where
\begin{equation}
B_{\nabla}:= \frac{\sqrt{p}B_r}{(1-\gamma)^2\lambda_{\min}\Sigma^{1/2}}.
    \end{equation}
  \end{lemma}
  \begin{proof}
    Use \eqref{eqn_nabla_U}, the fact that $\left|Q(s,a;h)\right|\leq B_r/(1-\gamma)$ (cf., \cite[Lemma 1]{paternain2018stochastic}) and that $\left\|\kappa(s,\cdot)\right\|= 1$ to upper bound the norm of $\nabla_h U_{s_0}(h)$ by
    \begin{equation}
      \left\|\nabla_hU_{s_0}(h)\right\|^2\leq \frac{B_r^2}{(1-\gamma)^4 }\mathbb{E}\left[\left\|\Sigma^{-1}(a-h(s))\right\|^2|h\right].
    \end{equation}
    Since the action a is drawn from a normal distribution with mean $h(s)$ and covariance matrix $\Sigma$ it follows that $\left\|\Sigma^{-1/2}(a-h(s))\right\|^2$ is a $\chi^2$ distribution and thus its expectation is $p$. Hence, the above expectation is bounded by $p\lambda_{\min}(\Sigma)^{-1}$. This completes the proof of the result. 
    \end{proof}
  %
  We are now in conditions of introducing the convergence of the online policy gradient algorithm presented in Section \ref{sec_online} to a neighborhood of the critical points of the value functional $U_{s_0}(h)$. In addition, the update is such that it guarantees that the model order remains bounded for all iterations. 
  %
  \begin{theorem}\label{theo_convergence}
    Let Assumptions \ref{assumption_reward_function}--\ref{assumption_large_h} hold. For any $\varepsilon>0$ chose $K$ such that 
    \begin{equation}\label{eqn_compression_budget}
      K<\frac{\varepsilon}{2B_{\nabla}}\frac{\beta_{\rho}}{B_\rho},  
    \end{equation}
algorithm step-size $\eta>0$ such that
    \begin{equation}\label{eqn_step_max}
\eta\leq      \frac{\sqrt{C_1^2+4C_2\left(\frac{\varepsilon \beta_\rho}{2 B_\rho}-B_{\nabla}K\right)}-C_1}{2C_2}.
      \end{equation}
    and compression budget of the form  $\epsilon_{K} = K\eta$. Under the hypotheses of Lemma \ref{lemma_lipshitz}, and for any kernel such that $\Sigma_{\ccalH}$ verifies \eqref{eqn_kernel_condition}, the sequence of policies that arise from Algorithm \ref{alg_online_policy_gradient} satisfy that $\liminf_{k\to\infty}\left\|\nabla_hU_{s_0}(h_k)\right\|^2 <{\varepsilon}$. In addition, let $M_k$ be the model order of $h_k$, i.e., the number of kernels which expand $h_k$ after the pruning step KOMP. Then, there exists a finite upper bound $M^\infty$ such that, for all $k\geq 0$, the model order is always bounded as $M_k\leq M^\infty$.  
\end{theorem}
\begin{proof}
  Define the following sequence of random variables
  \begin{equation}
V_k=\left(U(h^\star)-U(h_{k})\right)\mathbbm{1}(k\leq N),
    \end{equation}
  with $\mathbbm{1}(\cdot)$ being the indicator function and $N$ the stopping time defined in \eqref{eqn_stopping_time}. We next work towards showing that $V_k$ is a non-negative submartingale. Because $U(h^\star)$ maximizes $U(h)$, $V_k$ is always non-negative. In addition $V_k\in \ccalF_k$ since $U(h_k)\in\ccalF_k$ and $\mathbbm{1}(k\leq N)\in \ccalF_{k}$. Thus, it remains to be shown that $\mathbb{E}[V_{k+1} | \ccalF_k] \leq V_k$. Notice that for any $k> N$ it follows that $\mathbbm{1}(k\leq N) = 0$ and hence $V_k = 0$. Thus we have that $V_{k+1}= V_k$ for all $k\geq N$. We are left to show that $\mathbb{E}[V_{k+1} | \ccalF_k] \leq V_k$ for $k\leq N$. Using the result of Lemma \ref{lemma_supermartingale_pars} we can upper bound $\mathbb{E}[V_{k+1} | \ccalF_k]$ as
  \begin{equation}
    \begin{split}
      \mathbb{E}[V_{k+1} | \ccalF_k] \leq V_k +{{\eta^2}}C_1+{{\eta^3}}C_2+K\eta \left\|\nabla_hU_{s_0}(h_k)\right\| \\
      -\eta\left<\nabla_hU_{s_0}(h_k),\nabla_hU_{s_k}(h_k) \right>_{\ccalH}.
    \end{split}
    \end{equation}
  Since we have that $\left\|\nabla_hU_{s_0}(h_k)\right\|^2 \geq \varepsilon$ by virtue of Assumption \ref{assumption_large_h} it follows that {$\left\|h\right\|_\ccalH \leq  B_h(\varepsilon)$.} 
  Therefore, there exists some Hilbert Space for which condition \eqref{eqn_kernel_condition} holds and the result of Theorem \ref{prop_all_gradients} implies that the inner product in the right hand side of the previous expression is lower bounded by $\varepsilon\beta_\rho/(2B_{\rho})$. In addition, the norm of $\nabla_hU_{s_0}(h_k)$ is bounded by virtue of Lemma \ref{lemma_bound_nabla_U}. Hence, the previous expression can be further upper bounded by 

    \begin{equation}\label{eqn_submartingale}
\begin{split}          
          \mathbb{E}[V_{k+1} | \ccalF_k] &\leq V_k +{\eta^2}C_1+\eta^3C_2+\eta K B_{\nabla}
      -\eta\frac{\varepsilon}{2}\frac{\beta_\rho}{B_\rho} \\
      &= V_k +{\eta}\alpha(K,\eta), 
    \end{split}
    \end{equation}
    where we define $\alpha(K,\eta)$ as
    \begin{equation}\label{eqn_alpha}
\alpha(K,\eta) := K B_{\nabla}-\frac{\varepsilon}{2}\frac{\beta_\rho}{B_\rho}+\eta C_1+\eta^2C_2.
    \end{equation}
    With the condition for the compression factor satisfying \eqref{eqn_compression_budget} we guarantee that the sum of the first two terms on the right hand side of \eqref{eqn_alpha}  is negative. The latter is sufficient to guarantee that the expression is negative for all $\eta$ satisfying \eqref{eqn_step_max}. This completes the proof that $V_k$ is a non-negative submartingale.  Thus, $V_k$  converges to random variable $V$ such that $\mathbb{E}[V]\leq \mathbb{E}[V_0]$ (see e.g., \cite[Theorem 5.29]{durrett2010probability}). Then, by unrolling \eqref{eqn_submartingale}  we obtain the following upper bound for the expectation of $V_{k+1}$ 
    \begin{equation}
      \mathbb{E}\left[V_{k+1}\right] \leq V_0 + \alpha \eta\mathbb{E}N. 
      \end{equation}
    Since $V_{k+1}$ is bounded the Dominated Convergence Theorem holds and we have that
        \begin{equation}
  \mathbb{E}[V]=   \lim_{k \to\infty} \mathbb{E}\left[V_{k+1}\right] \leq V_0 + \alpha \eta\mathbb{E}N. 
      \end{equation}
       Since $\alpha<0$, rearranging the terms in the previous expression we can upper bound $\mathbb{E}N$ by
        \begin{equation}
          \mathbb{E}N \leq \frac{\mathbb{E}{V}-V_0}{\eta \alpha}.
          \end{equation}
        Therefore it must be case that $P(N=\infty)=0$. Which implies that the event $\left\| \nabla_h U_{s_0}(h_k)\right\| < \varepsilon$ occurs infinitely often. Thus completing the proof of the result.  It remains to be shown that the model order of the representation is bounded for all $k$. The proof of this result is identical to that in \cite[Theorem 3]{koppel2016parsimonious}.

\end{proof}
The previous result establishes the convergence to a neighborhood of the critical points of $U_{s_0}(h)$ of the online gradient ascent algorithm presented in Section \ref{sec_online}. For such result to hold, we require that the kernel width, the compression budget and the step size to be small enough. In addition, the compression introduced by KOMP guarantees that the model order of the function $h_k$ remains bounded for all $k\geq 0$. In the next section we explore the implications of these theoretical results in {a cyclic navigation problem.}
%


\section{Numerical Experiments}\label{sec_numerical}

Next we test the performance of our non-epsodic RL method in a suvelliance and navigation task.   The setup includes an area in $\mathbb R^2$ with a point to be surveilled located at $x_g=[-1,-5]$ and a battery charger located at $x_b=[-1,5]$. These points are depicted in green and red, respectively, in  Figure \ref{fig:trayectoria}.  An agent  starts moving from its initial point at $ x_0=[3,0]$, depicted in blue in figures \ref{fig:trayectoria} and \ref{fig:fases},  towards its goal at $x_g$. The agent model   consist in  second order point mass acceleration dynamics, with position $x\in\mathbb R^2$ and velocity $v\in\mathbb R^2$ as state variables. It also includes second order battery charging and discharging equations with state variables $b\in \mathbb R$ modeling the remaining charge of the battery, and $d\in \mathbb R$ representing the difference between charge levels at two consecutive time instants. The battery charges at a constant rate $\Delta B$   if the agent is located within a neighborhood of the battery charger, and discharges at the same rate $\Delta B$ otherwise. Vector  $s=(x,v,b,d)$ collects all the state variables of this model.
The reward is shaped so that the agent is stimulated to move towards the goal $x_g$ when $b$ is grater than or equal to $40\%$ of its full capacity and it is discharging $d<0$.  And towards the battery charger $x_c$ if $b$ is lower than $40\%$ and discharging $d<0$, or if $b$ is lower than $90\%$ and charging $d>0$. The use of a second order model for the battery allow us to leave room for some hysteresis on the charging and discharging loop, so that the battery does not start discharging as soon as $b$ surpasses the $40\%$ level. Instead, it keeps charging until it reaches $90\%$ of its capacity before moving back towards the goal.   
A logarithmic barrier is added to the reward for helping the agent to avoid an elliptic obstacle centered at $x_0=[0,0]$ with horizontal and vertical axes of length $1.8$ and $0.9$, respectively 
Under these dynamics, the agent decides its acceleration $a_k\in\mathbb R^2$ using the randomized Gaussian policy $a_k\sim \pi_{h_k}(.|s_k)$ where the mean of $a_k\sim \pi_{h_k}(.|s_k) $ is the kernel expansion $h_k(s_k)$ updated via \eqref{eqn_stochastic_update}. The $Q$-function in step 7 of Algorithm \ref{alg_stochastic_gradient} is estimated as the sum of $T_Q$ consecutive rewards with $T_Q$ drawn from a geometric distribution of parameter $\gamma$. The Gaussian noise $n_k$ that is added to $h_k(s_k)$  in step 2 of Algorithm \ref{alg_stochastic_grad} is selected at random  when $k$ is even, and equal to $n_k=-n_{k-1}$ when $k$ is odd. This is a practical trick to improve the ascending direction of the stochastic gradient without adding bias or violating the  Gaussian model for the randomized policy $\pi_{h}(a|s)$, see \cite{paternain2018stochastic} for more details.    

Figure \ref{fig:trayectoria} shows the trajectory  of the agent, with its color changing gradually from blue to red as it starts from $x_0$  and  loops between $x_g$ and $x_b$. 
\begin{figure}     
\includegraphics[scale=0.2]{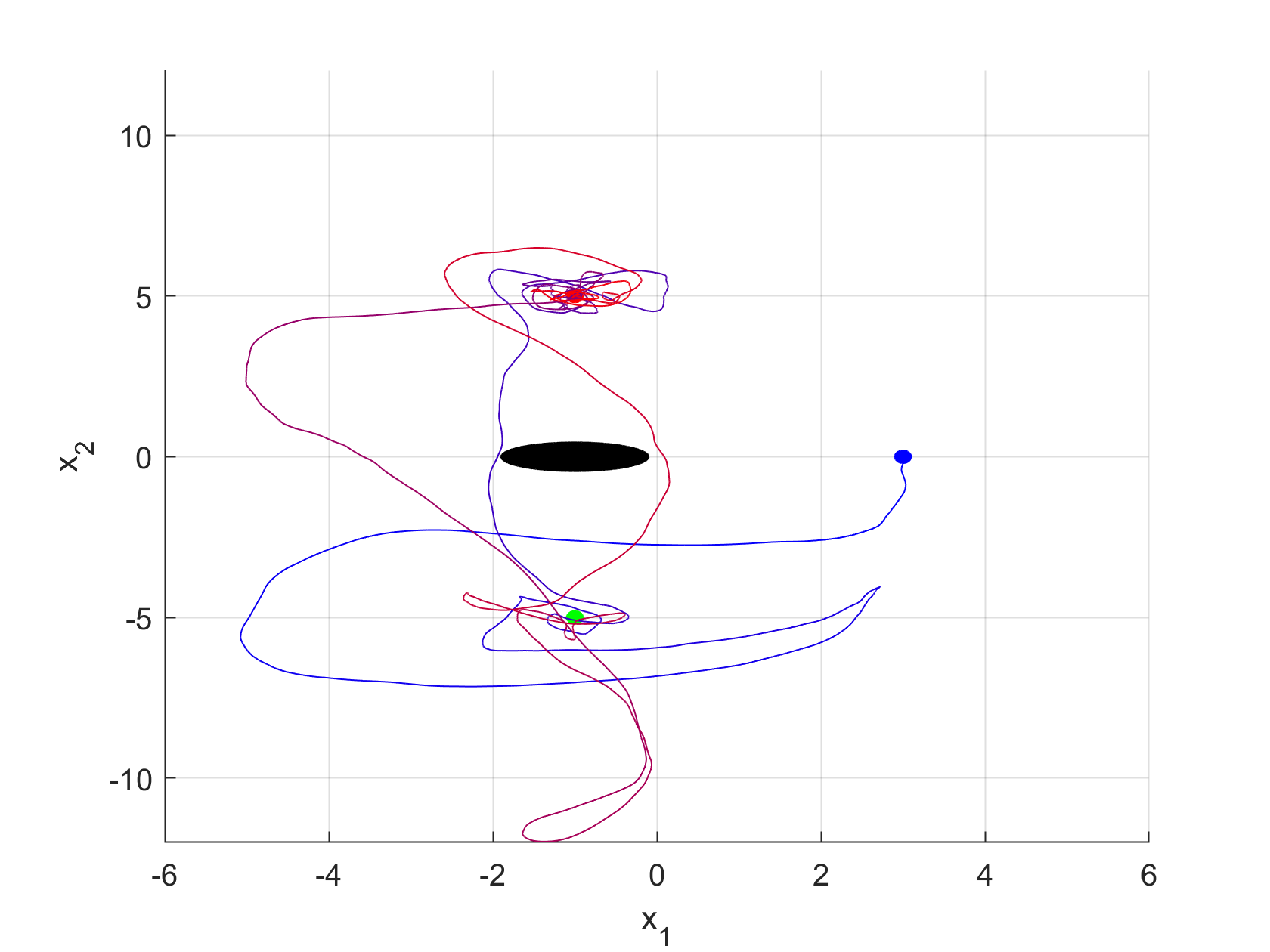}
\caption{Online cycling trajectory of an agent starting at $x_0=[3,0]$ with the goal of  surveilling the location represented by a green point  at $x_b=[-1,-5]$. The agent  needs to recharge its battery when it discharges below $40\%$ of its maximum capacity, at the charger location $x_b=[-1,5]$ represented by the red point. When navigating towards $x_b$ the agent must  avoid the ellipsoidal obstacle centered at  $[-1,0].$ }\label{fig:trayectoria}
\end{figure}
The four stages of this looping trajectory are detailed in Figure \ref{fig:fases}, with Figure \ref{fig:fases} (left) showing the trace from $x_0$ to the neighborhood of $x_g$ which includes some initial exploring swings. Then the trajectory in \ref{fig:fases} (center left) starts when the agent's battery crosses the threshold of $40\%$.
 In this case the agent is rewarded for moving towards the charger and staying in its neighborhood until $b$ reaches $90\%$. The next stage in Figure \ref{fig:fases} (center right) starts when the battery level reaches $90\%$ and the agent moves back to the goal. And finally the trace in Figure \ref{fig:fases} (right) starts when the battery discharges under the safety level of $40\%$,  moving back towards the charger and closing the loop.     

\begin{figure*}
\centering
  \includegraphics[scale=0.1]{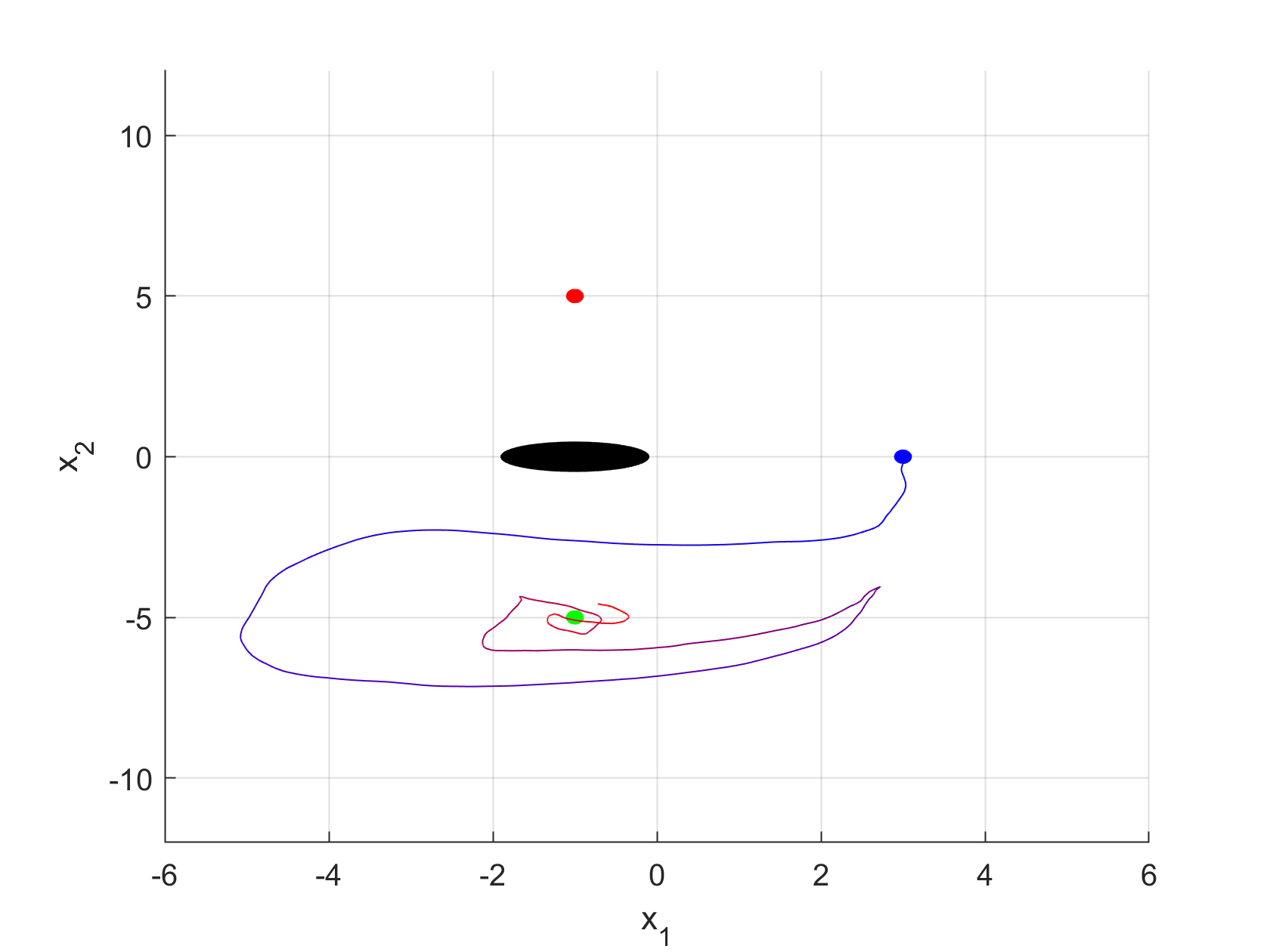}
\includegraphics[scale=0.1]{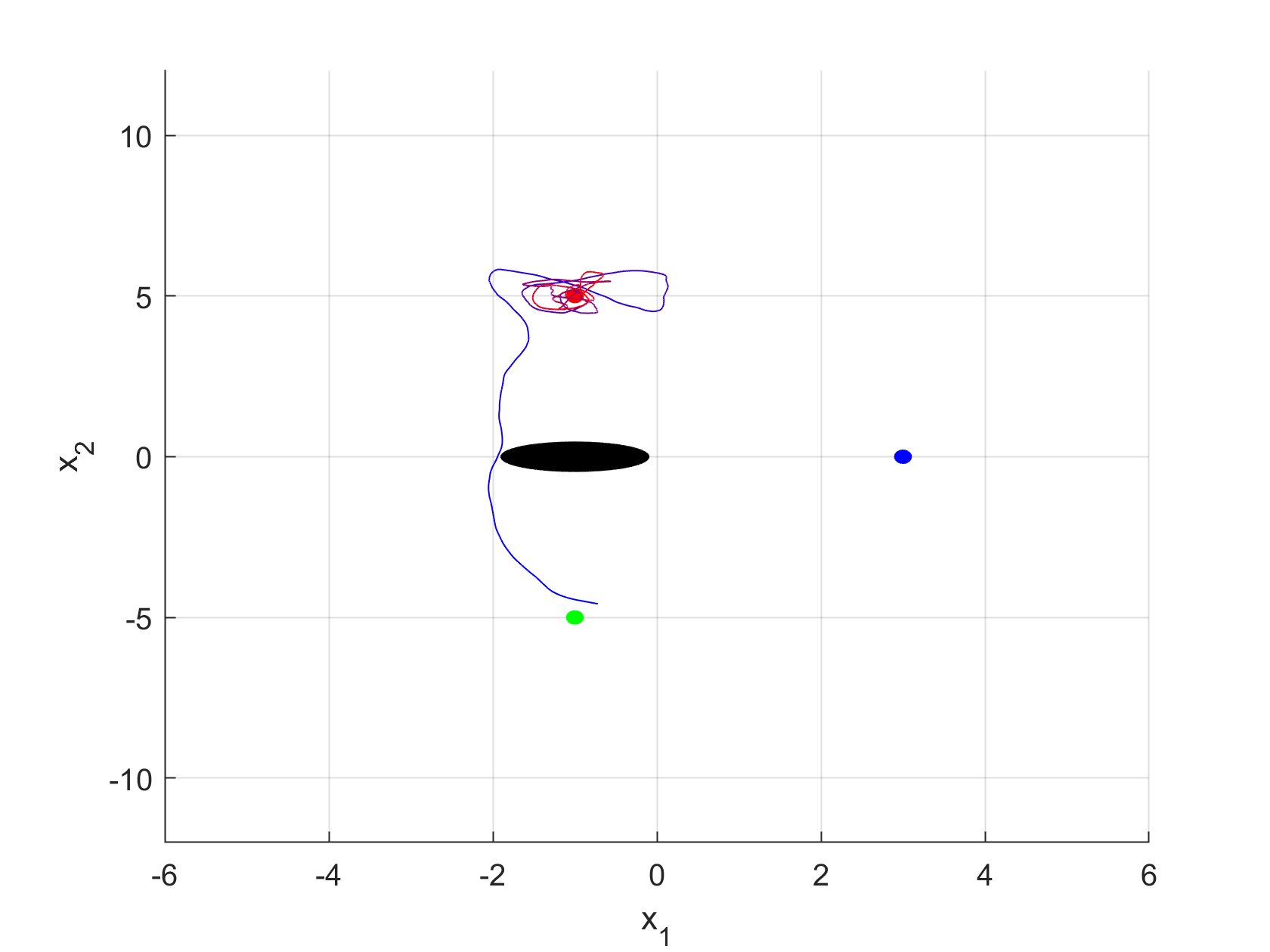}
\includegraphics[scale=0.1]{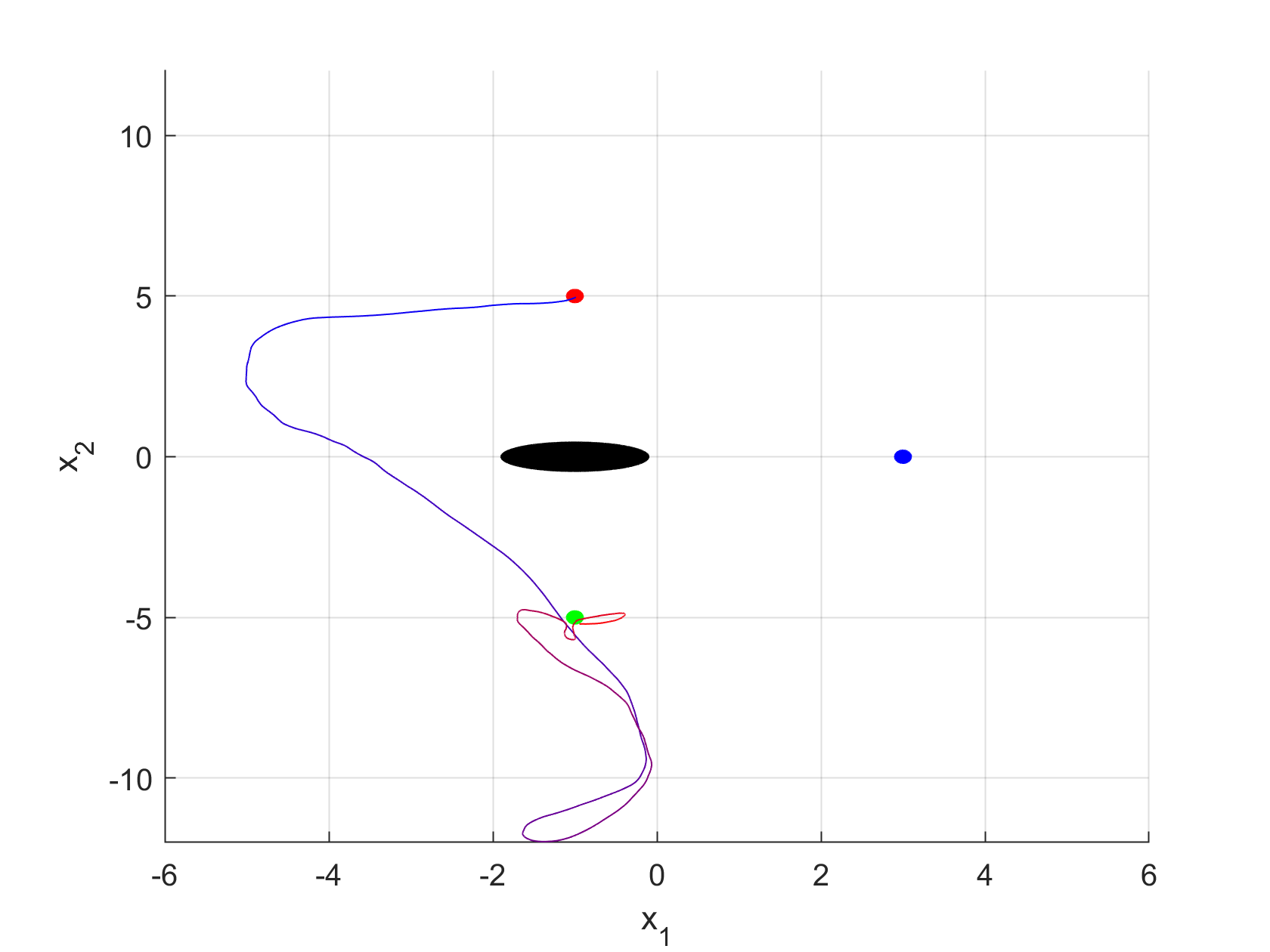}
\includegraphics[scale=0.1]{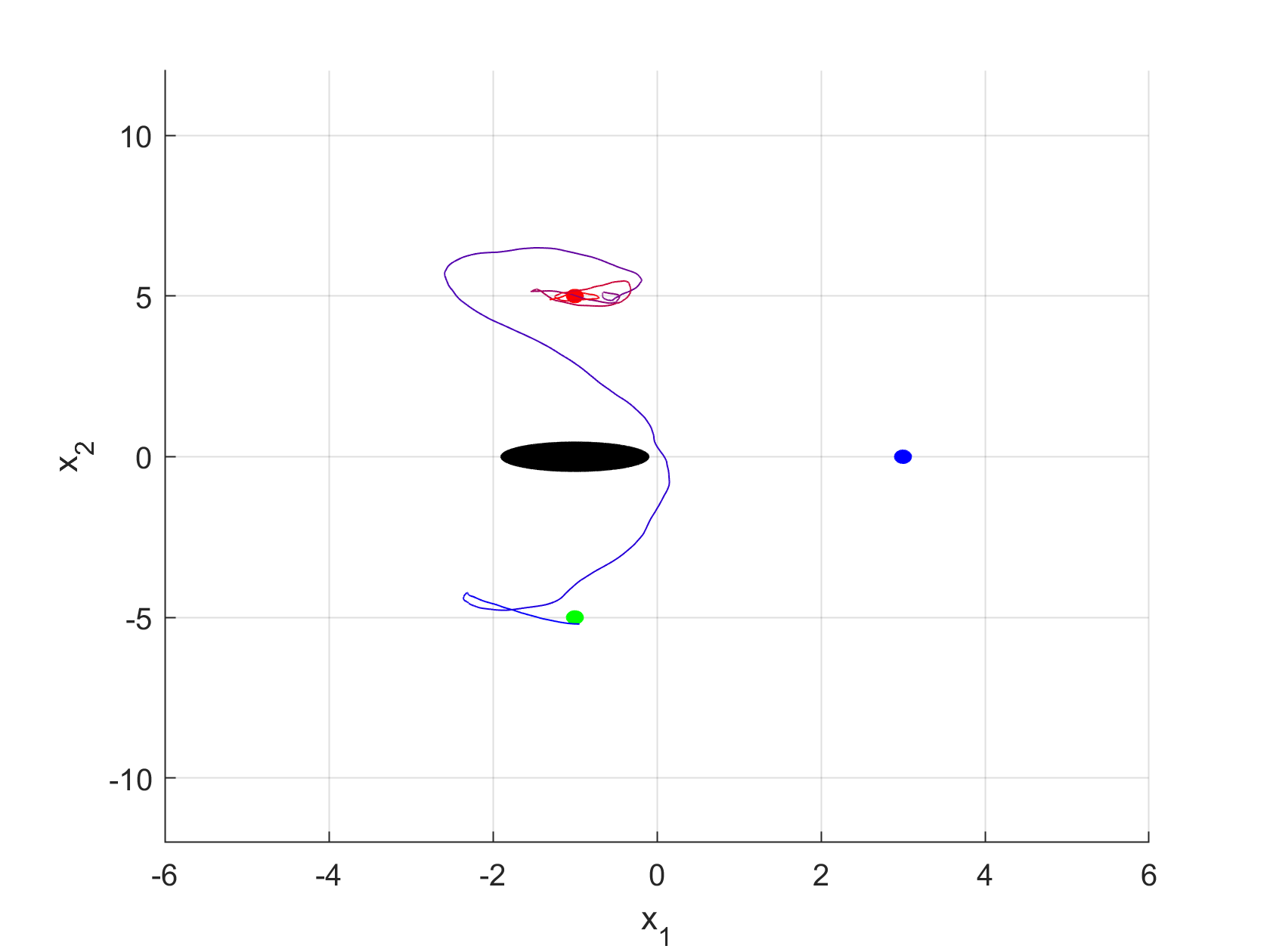}
\caption{Detail of the four stages of the online trajectory in Figure \ref{fig:trayectoria}. Trace from $x_0$ to the neighborhood of $x_g$ (left). Trace when the agent's battery crosses the threshold of $40\%$ (center left). Trace when the battery level reaches $90\%$ and the agent moves back to the goal (center right). Trace when the battery discharges under the safety level of $40\%$ (right). }\label{fig:fases}
\end{figure*}

This coherence between the battery level and the agent trajectories is further illustrated in figures \ref{fig:step} and \ref{fig:cycle}. Figure \ref{fig:step}   depicts the battery level across iterations in time, alongside  the vertical position of the agent. The vertical position  shows {an oscillating step-response like behavior,} as the agent reaches the neighborhood of the goal and the charger sequentially and hovers around them. The battery level shows the desired hysteresis, transitioning according to the thresholds at $40\%$ and $90\%$ and depending on the charging slope.

\begin{figure}  
\includegraphics[scale=0.2]{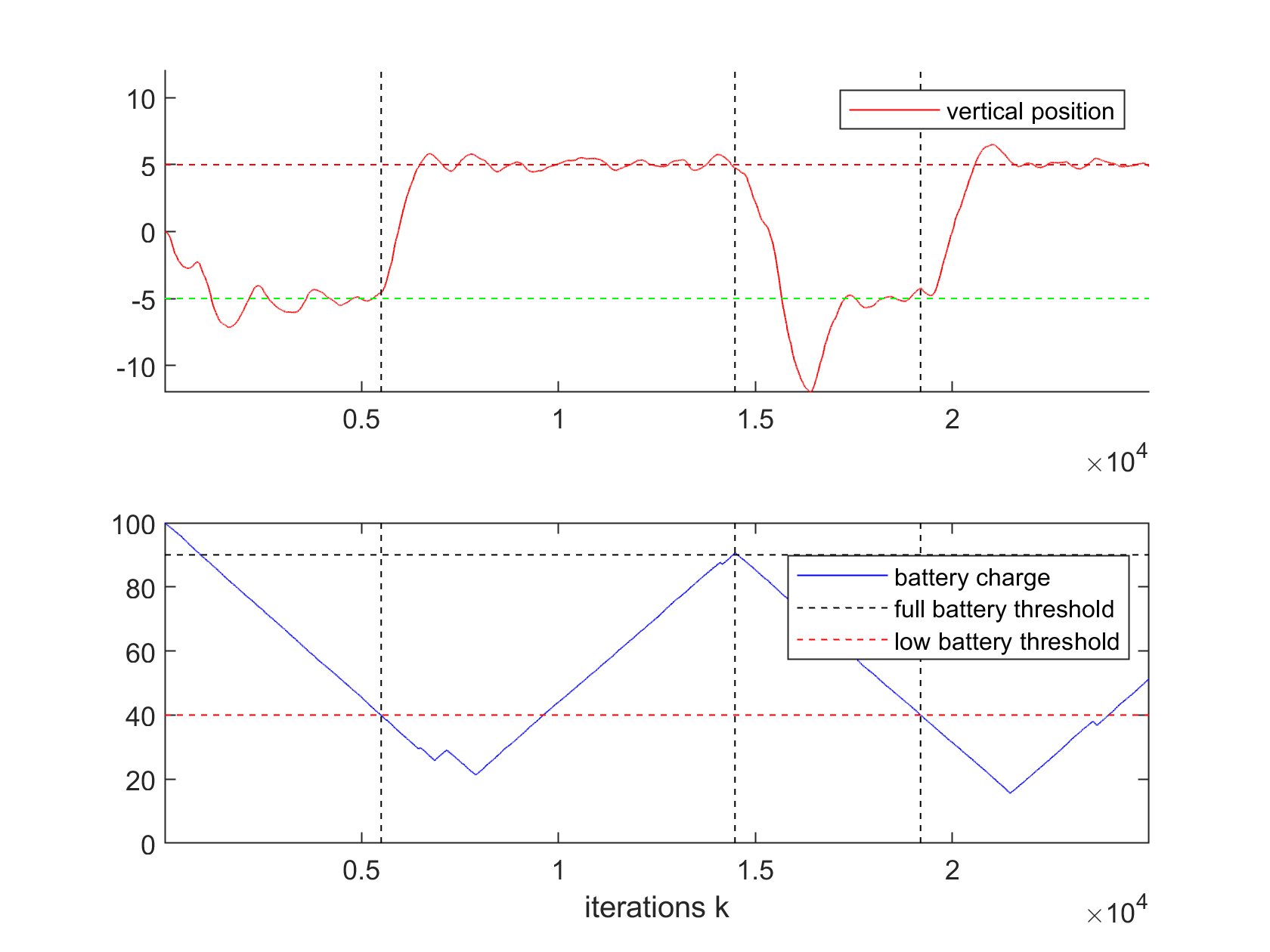} 
\caption{Vertical position and battery charge as a function of the online gradient update index $k$. The horizontal dashed lines correspond to the  positions of the goal and the charger, and to the battery safety and fully charged thresholds. The vertical dashed lines mark the transitions when the agent is directed to move to the charger or to the goal.  } \label{fig:step}
\end{figure}

  Such a loop  is further evidenced in Figure \ref{fig:cycle}, which represents the agent's  vertical position versus its battery level. The horizontal dashed lines represent the full charge and low battery thresholds, and the vertical dashed lines correspond to the  positions of  the goal and the charger. This figure shows how starting from $x_0$ the agent moves towards $x_g$ until  it reaches the $40\%$ level, and  then towards the charger  hovering around it until the battery charge is  $90\%$. Then it closes the loop by moving towards the goal and the charger sequentially   while its battery charges and discharges.   
\begin{figure}  
  \includegraphics[scale=0.2]{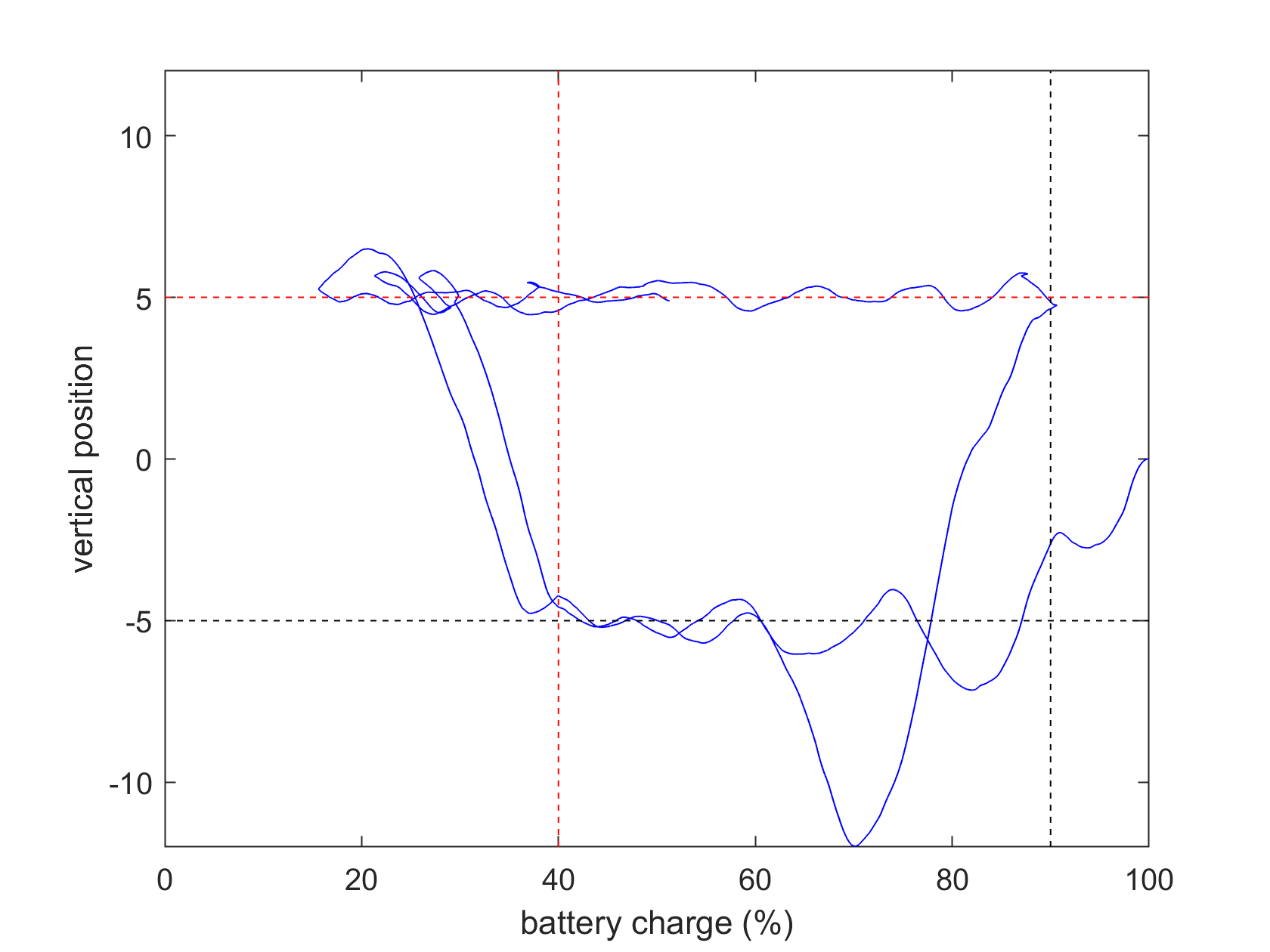} 
  \caption{Cyclic evolution of the agent's vertical position versus the battery charge. Horizontal dashed lines correspond to the  positions of the goal and the charger, and vertical dashed lines represent the battery safety and fully charged thresholds. }\label{fig:cycle}
\end{figure}

Figure \ref{fig:rewards} is included  to corroborate the theoretical findings of Section \ref{sec_important_properties}. More specifically, it depicts the evolution of $U_s(h_k)$ as a function of the online iteration index $k$. 
%
 %
  The starting point $s$  in $U_s(h_k)$ is the same for all $k$, and it is  selected as the state when the battery level crosses the line of $40\%$ for the first time. It corresponds to the location near $x_g$ where the stage 2 starts in Figure \ref{fig:fases}(b).  For completeness  $s=[x,v, b,d]$ with {$x=[-0.72,-4.58]$, $v=[-0.092,0.049]$, $b=39.99$, and $d=3\times 10^{-4}$.} The value function  $U_{s}(h_k)$  is estimated using rewards obtained by an episodic agent that starts at $s$ and runs  $N=100$ sample trajectories of length $T$ selecting actions according to the policy $h_k$, which is kept constant during the $T$ state transitions. These rewards are averaged according to
  \begin{equation}
  \hat U_s(h_k)=\frac{1}{N}\sum_{i=1}^N R_{ik}
  \label{eq:hatushk}
  \end{equation}
where $R_{ik}=\sum_{t=0}^T \gamma^t r_{itk}$ and $r_{itk}$ is the instantaneous reward obtained by the episodic agent at time  $t$, using policy $h_k$, over the sample trajectory $i$. 
These episodic trajectories are carried out for assessing performance of a fixed $h_k$, but the algorithm for updating these policies is  non-episodic, as it evolves in the fully online fashion of  \eqref{eqn_stochastic_update}.


 The horizon $T=100$ for these episodes was  selected  so that that the discarded tail of the geometric series becomes  negligible, with $\gamma^T\simeq 2\times 10^{-5}$ staying  under the noise deviation.   It is  remarkable  that when the online agent travels through $s$ on its online journey of Figure \ref{fig:trayectoria}, the policy figures out  how to increase the reward.  And  such a reward will not decrease when the policy is updated in the future. This is coherent with our theoretical findings in Theorem \ref{prop_all_gradients}, which states that gradients at future states are ascent directions for the value function  at a previous state, that is  $s$ in this case.
  
  As stated before, each point on the blue line in Figure \ref{fig:rewards} represents the mean of rewards in \eqref{eq:hatushk}, and it is accompanied by its deviation interval. 
   Notice  that, even if the improvement in reward is relative minor, at $0.3\%$, it  is good enough to direct the agent towards the battery charger. This can be better seen in the next figure. 
\begin{figure}  
  \includegraphics[scale=0.2]{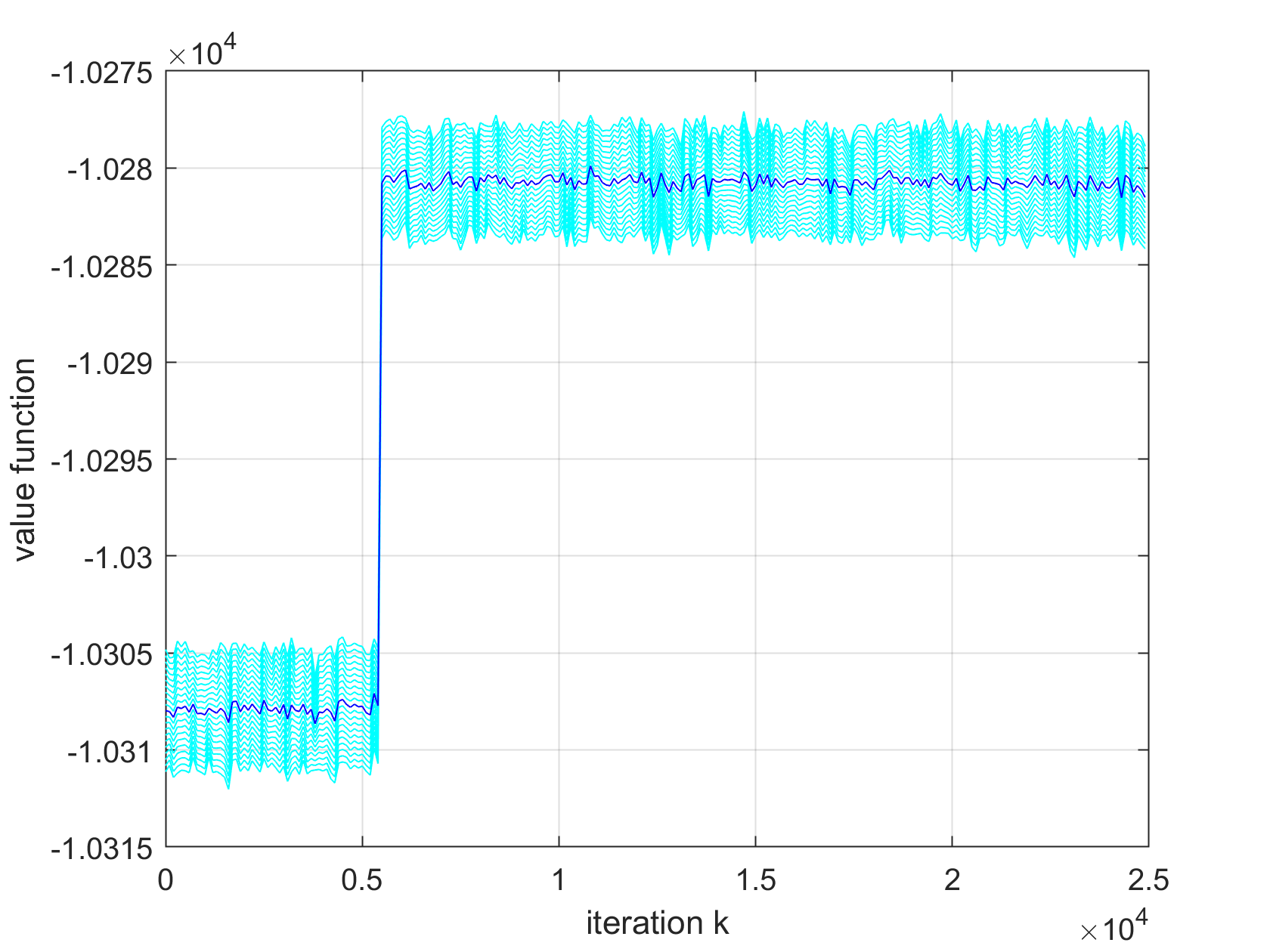} 
  \caption{Evolution of the mean accumulative reward in \eqref{eq:hatushk} as a function of the online iteration step $k$.}  \label{fig:rewards}
\end{figure}

  Figure \ref{fig:trayectoriaColor} shows five different trajectories starting at the same point $x=[-0.72,-4.58]$ represented by a blue dot.  The trajectory that passes through the colored dots corresponds to an agent running our online algorithm, and coincides with part of the second stage in Figure \ref{fig:fases} (center left). Let  $k_0$, $k_1$, $k_2$, and $k_3$ be the iteration indexes when the online agent reaches the points $x_{k_0}=x$, $x_{k_1}$, $x_{k_2}$, $x_{k_3}$, represented by the blue, cyan, green, and purple dots, respectively.  At these iterations the agent produces policies $h_{k_0}$, $h_{k_1}$, $h_{k_2}$, and $h_{k_3}$. The blue, cyan, green, and purple lines in Figure \ref{fig:trayectoriaColor} represent the trajectories of an episodic agent starting at $x$ and navigating with constant policies  $h_{k_0}$, $h_{k_1}$, $h_{k_2}$, and $h_{k_3}$, respectively.  Figure  \ref{fig:trayectoriaColor} corroborates that the policies improve  as the online agent moves along its trajectory, allowing the episodic agent to navigate better.   
%
%
 Indeed, at first  the episodic agent only knows to go north west on the straight blue line, but eventually it manages to follow the purple line moving towards the charger and avoiding the  obstacle.  This apparent improvement in Figure \ref{fig:trayectoriaColor} is not reflected in a significant step increase  in  Figure \ref{fig:rewards}. This is because the forgetting factor $\gamma=0.9$ weights a few steps of the trajectory in the value function, and the fist steps are where the trajectories are not significantly separated.

  
  Overall, this numerical example shows that the algorithm developed in this paper is capable of learning how to navigate on a loop in between to goal locations, avoiding an obstacle, and following a cyclic trajectory that does not comply with the standard stationary assumptions in the literature.   
\begin{figure}  
  \includegraphics[scale=0.2]{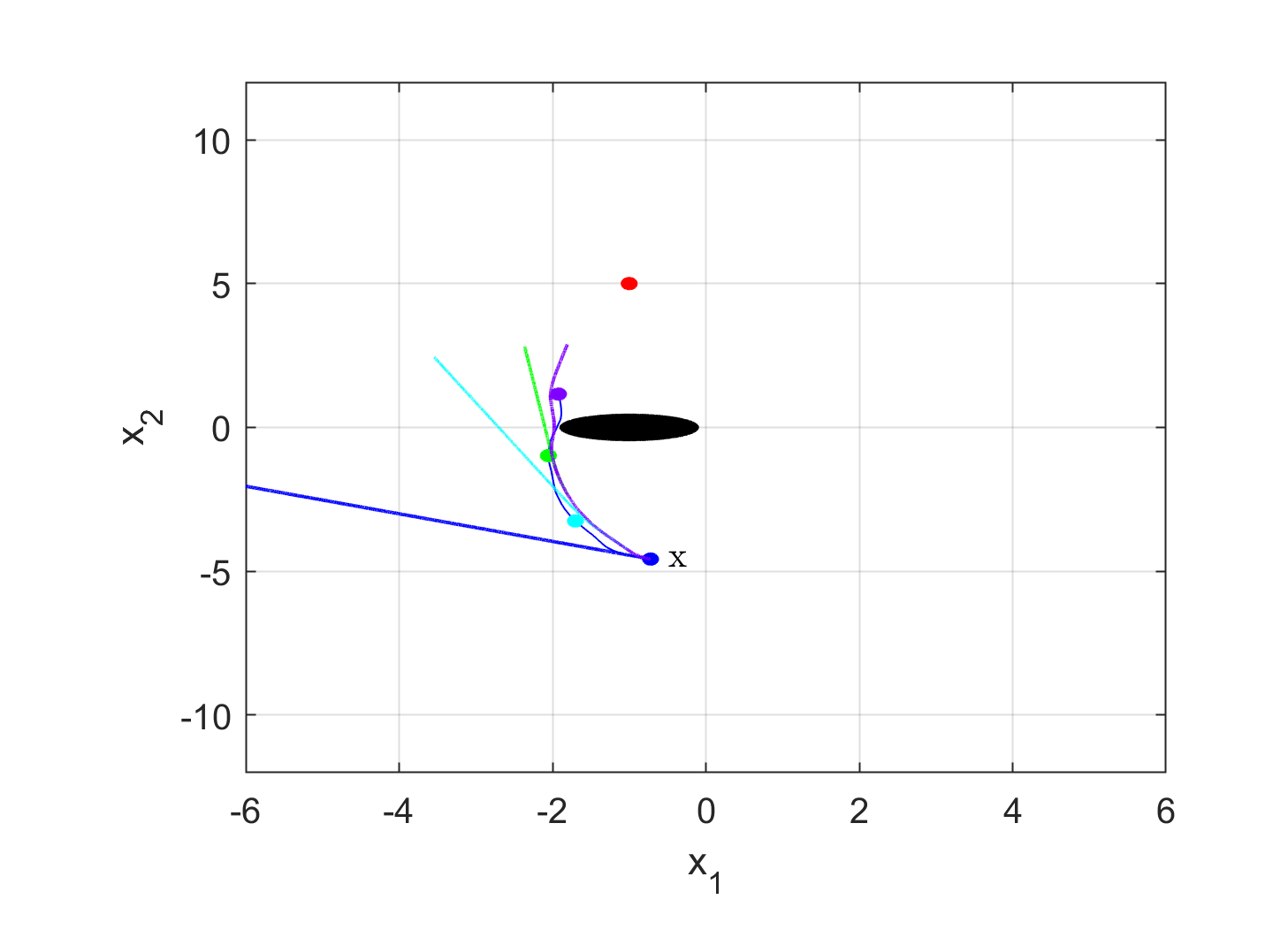}
  \caption{Trajectories of an episodic agent using four   policies that are produced by the online agent when following the cyclic trajectory of Figure \ref{fig:trayectoria}. Each colored point represents a location $x_k$ in which the online agent updates policy to obtain  $h_k$. The line of the corresponding color represents the trajectory of the episodic agent that uses  the fixed  policy  $h_k$ to navigate from $x$  towards the charger.   }  \label{fig:trayectoriaColor}
\end{figure}


\section{Conclusion}
We have considered the problem of learning a policy that belongs to a RKHS in order to maximize the functional defined by the expected discounted cumulative reward that an agent receives. In particular, we presented a fully online algorithm that accumulates at the critical points of the value function and keeps the model order of the representation of the function bounded for all iterations. The algorithm uses unbiased estimates of the gradient of the functionals conditioned at the current state that can be achieved in finite time. We establish that these gradients are also ascent directions for the initial value function for Gaussian kernels with small enough bandwidth. Therefore, by updating the policy following such gradients the value of the initial value function is increased in expectation at each iteration, when the step size and compression budget are small enough. { We tested this algorithm in a navigation and surveillance problem  whose cyclic nature  highlights the ability  to operate in a non stationary setup. The surveillance task is carried out while by training in a fully online fashion, without the need of episodic restarts. With this experiment we also corroborated our claim in Theorem \ref{theo_convergence} regarding the ascent directions of the stochastic gradients.}

\appendix


%

 \subsection{Proof of Lemma \ref{lemma_distribution}}\label{appendix_lemma_distribution}

\begin{proof}
  Since $s\in\ccalS_0$ there exists some time $t\geq 0$ such that $p(s_t=s|s_0)> 0$. Therefore we have that
  \begin{equation}
    \begin{split}
      \rho_{s_0}(s) &= (1-\gamma)\sum_{u=0}^\infty \gamma^u p(s_u=s|s_0) \\
      &\geq (1-\gamma)\gamma^{t}p(s_t=s|s_0),
      \end{split}
    \end{equation}
  where the last inequality follows from the fact that $\gamma>0$ and $p(s_u=s|s_0) \geq 0$ for all $u\geq 0$. Hence, by assumption it follows that $\rho_{s_0}(s) >0$. To prove the second claim, start by writting $\rho_{s^\prime}(s^{\prime\prime})$ as
  \begin{equation}
    \begin{split}
            \rho_{s^\prime}(s^{\prime\prime}) &= (1-\gamma)\sum_{t=0}^\infty \gamma^t p(s_t=s^{\prime\prime}| s_0=s^\prime) \\
      &=  (1-\gamma)\gamma^{-T}\sum_{u=T}^\infty  \gamma^u p(s_u=s^{\prime\prime}| s_T=s^\prime) ,
      \end{split}
  \end{equation}
  where the last equality holds for any $T\geq 0$. Using the Markov property for any $u$ we have that
  \begin{equation}
p(s_u=s^{\prime\prime}| s_T=s^\prime) = p(s_u=s^{\prime\prime}| s_T=s^\prime,s_0).
    \end{equation}
  %
    %
    Since $s^\prime \in S_0$, there exists $T\geq 0$ such that $p(S_T=s^\prime |s_0)>0$. For that specific $T$, we have that
    \begin{equation}
      p(s_u=s^{\prime\prime}| s_T=s^\prime,s_0) = \frac{p(s_u=s^{\prime\prime}, s_T=s^\prime|s_0)}{p(s_T=s^\prime|s_0)}.
    \end{equation}
    Notice next that since $s^{\prime\prime} \in \ccalS\setminus \ccalS_0$ we have that $p(s_u=s^{\prime\prime}|s_0)=0$ for all $u\geq 0$. Hence, we also have that $p(s_u=s^{\prime\prime},S_T=s^\prime|s_0)=0$ for all $u\geq 0$ which completes the proof of the proposition.
  \end{proof}
  \subsection{Proof of Lemma \ref{lemma_lipshitz}}\label{appendix_lipshitz}
    Without loss of generality, we prove the result for $D(s)\rho_{s_0}(s)$.  We start by showing that the cumulative weighted distribution $\rho_{s_0}(s)$ is Lipschitz with constant $L_p$. 
  \begin{lemma}
    Under Assumption \ref{assumption_prob_dist}, the distribution $\rho_{s_0}(s)$ is Lipschitz with constant $L_p$, where $L_p$ is the Lipschitz constant defined in Assumption \ref{assumption_prob_dist}. 
  \end{lemma}
\begin{proof}
  Let us start by writing $p(s_t=s|s_0)$ by marginalizing it
    \begin{equation}
  \begin{split}
&    p(s_t=s|s_0) =\int p(s_t=s,a_{t-1},s_{t-1}|s_0)\, ds_{t-1} da_{t-1}.
    \end{split}
  \end{equation}
    Using Bayes' rule and the Markov property of the transition probability it follows that 
\begin{equation}
  \begin{split}
&    p(s_t=s|s_0) =\\ &\int p(s_t=s|a_{t-1},s_{t-1})\pi_h(a_{t-1}|s_{t-1})p(s_{t-1}|s_0)\, ds_{t-1} da_{t-1}.
    \end{split}
  \end{equation}
Using the Lipschitz property of the transition probability (cf., Assumption \ref{assumption_prob_dist}) we can upper bound $\left| p(s_t=s|s_0) - p(s_t=s^\prime|s_0)\right|$ by
\begin{equation}
  \begin{split}
    \left| p(s_t=s|s_0) - p(s_t=s^\prime|s_0)\right| \leq L_p\left\|s-s^\prime \right\| \\
        \int \pi_h(a_{t-1}|s_{t-1}) p(s_{t-1}|s_0) \, ds_{t-1}da_{t-1}  \\
    \end{split}
  \end{equation}
Since both $\pi_h(a_{t-1}|s_{t-1})$and $p(s_{t-1}|s_0)$ are probability distributions they integrate one. Thus we have that
\begin{equation}\label{eqn_lipschitz_p}
  \begin{split}
    \left| p(s_t=s|s_0) - p(s_t=s^\prime|s_0)\right| \leq L_p\left\|s-s^\prime \right\|.
    \end{split}
  \end{equation}
Use the definition of $\rho_{s_0}(s)$ (cf., \eqref{eqn_discounted_distribution}) to write the difference $\left|\rho_{s_0}(s)-\rho_{s_0}(s^\prime)\right|$ as
\begin{equation}
  \begin{split}
    \left|\rho_{s_0}(s)-\rho_{s_0}(s^\prime)\right|=\\
    (1-\gamma)\left|\sum_{t=0}^\infty\gamma^t \left(p(s_t=s|s_0)-p(s_t=s^\prime|s_0)\right)\right|.
  \end{split}
\end{equation}
Use the triangle inequality to upper bound the previous expression by
\begin{equation}
  \begin{split}
    \left|\rho_{s_0}(s)-\rho_{s_0}(s^\prime)\right|&\\
    \leq (1-\gamma)&\sum_{t=0}^\infty\gamma^t\left|p(s_t=s|s_0)-p(s_t=s^\prime|s_0)\right|.
  \end{split}
\end{equation}
By virtue of \eqref{eqn_lipschitz_p}, each term can be upper bounded by $\gamma^tL_p\left\|s-s^\prime\right\|$. Thus
\begin{equation}
  \begin{split}
    \left|\rho_{s_0}(s)-\rho_{s_0}(s^\prime)\right|
    \leq (1-\gamma)\sum_{t=0}^\infty\gamma^tL_p\left\|s-s^\prime\right\|=L_p\left\|s-s^\prime\right\|.
  \end{split}
\end{equation}
This completes the proof of the lemma.
 \end{proof}

\begin{lemma}
  Under the Assumptions of Lemma \ref{lemma_lipshitz} $D(s)$ is bounded by $B_D$ (cf., \eqref{eqn_bound_D}) and it is Lipschitz, i.e.,
  \begin{equation}
    \left\|D(s)-D(s^\prime)\right\| \leq L_D\left\|s-s^\prime\right\|,
\end{equation}
  were $L_D$ is the constant defined in Lemma \ref{lemma_lipshitz}. 

  \end{lemma}
\begin{proof}
Let us start by introducing the change of variables $\zeta=\Sigma^{-1/2}(a-h(s))$ to compute $D(s)$. Hence we have that
    \begin{equation}
      D(s) =\int Q(s,h(s)+\Sigma^{1/2}\zeta) \frac{\zeta}{\sqrt{2\pi}^p} e^{-\left\|\zeta\right\|^2/2} \,d\zeta.
      \end{equation}
    Notice that we can define $\phi(\zeta) = e^{-\left\|\zeta\right\|^2/2}/\sqrt{2\pi}^p$ and then, the previous expression reduces to
       \begin{equation}\label{eqn_D_change_variable}
      D(s) =-\int Q(s,h(s)+\Sigma^{1/2}\zeta)  \nabla\phi(\zeta) \,d\zeta.
      \end{equation}
       Thus, integrating each component of $D(s)$ by parts we have for each $i=1,\ldots, n$ that
       \begin{equation}\label{eqn_parts_integration}
         \begin{split}
           D(s)_i = \int Q(s,h(s)+\Sigma^{1/2}\zeta)  \phi(\zeta) \Big|_{\zeta_i-\infty}^{\zeta_i=\infty} d\bar{\zeta}_i\\
           +\int \frac{\partial Q(s,h(s)+\Sigma^{1/2}\zeta)}{\partial \zeta_i} \phi(\zeta)\, d\zeta,
           \end{split}
      \end{equation}
       where $\bar{\zeta}_i$ denotes the integral with respect to all variables in $\zeta$ except for the $i$-th component. Since $Q(s,a)\leq B_r/(1-\gamma)$ \cite[Lemma 3]{paternain2018stochastic} and $\phi(\zeta)$ is a Multivariate Gaussian density the first term in the above sum is zero. Next we compute the derivative of the $Q$-function with respect to $\zeta$. By the chain rule we have that 
       \begin{equation}
         \frac{\partial Q(s,h(s)+\Sigma^{1/2}\zeta)}{\partial \zeta} = \Sigma^{1/2}\frac{\partial Q(s,a)}{\partial a}\Big|_{a=h(s)+\Sigma^{1/2}\zeta}.  
       \end{equation}
       Thus, \eqref{eqn_parts_integration} reduces to
              \begin{equation}
         \begin{split}
           D(s) =\int \frac{\partial Q(s,a)}{\partial a}\Big|_{a=h(s)+\Sigma^{1/2}\zeta} \phi(\zeta)\, d\zeta.
           \end{split}
      \end{equation}
              We claim that the first term in the above integral
              \begin{equation}\label{eqn_q_prime_def}
                Q^\prime(s):= \frac{\partial Q(s,a)}{\partial a}\Big|_{a=h(s)+\Sigma^{1/2}\zeta}
              \end{equation}
              is Lipschitz with constant $L_{D}$. That being the case, one has that
              \begin{equation}
                \begin{split}
                  \left\| D(s)-D(s^\prime)\right\| &\leq \int\left\| Q^\prime(s) -Q^\prime(s^\prime)\right\|\phi(\zeta)\,d\zeta \\
                  &\leq L_{D}\left\| s -s^\prime\right\| \int\phi(\zeta) \, d\zeta =  L_{D}\left\| s -s^\prime\right\|.   
                  \end{split}
                \end{equation}
              Thus, to show that $D(s)$ is Lipschitz with constant $L_D$ it remains to be showed that $Q^\prime(s)$ is Lipschitz with the same constant. To do so, write the $Q$ function as
    \begin{equation}
Q(s,a) = r(s,a) + \sum_{t=1}^\infty \gamma^t \int r(s_t,a_t)p(s_t|a,s)\pi_h(a_t|s_t) \, ds_tda_t,
      \end{equation}
and compute its derivative with respect to $a$
    \begin{equation}\label{eqn_aux_q_prime}
      \begin{split}
        Q^\prime(s,a)=\frac{\partial Q(s,a)}{\partial a} = \frac{\partial r(s,a)}{\partial a} \\+ \sum_{t=1}^\infty \gamma^t \int r(s_t,a_t)\frac{\partial p(s_t|a,s)}{\partial a}\pi_h(a_t|s_t) \, ds_tda_t.
        \end{split}
      \end{equation}
    Since $\partial r(s,a)/\partial a$ is Lipschitz in both arguments (cf., Assumption \ref{assumption_reward_function}), to show that the derivative of $Q$ is Lipschitz, it suffices to show that $\partial p(s_t|a,s)/\partial a$ is Lipschitz as well. To that end, write $p(s_t|s,a)$ as
    \begin{equation}
      \begin{split}
        p(s_t|s,a)= \\\int p(s_1|s,a)\prod_{u=1}^{t-1}\pi_h(a_u|s_u)p(s_{u+1}|s_u,a_u) \, d\bbs_{t-1} d\bba_{t-1},
        \end{split}
    \end{equation}
    where $d \bbs_{t-1} = (ds_1, \cdots, ds_{t-1})$ and $d\bba_{t-1}=(da_1,\cdots, da_{t-1})$. Let us define
    \begin{equation}
      \begin{split}
        \Delta p_t(s^\prime,s^{\prime\prime},a^\prime,a^{\prime\prime}):=   \frac{ \partial p(s_t|s^\prime,a)}{\partial a}\big|_{a=a^\prime} -\frac{ \partial p(s_t|s^{\prime\prime},a)}{\partial a}\big|_{a=a^{\prime\prime}} .
        \end{split}
      \end{equation}
    Using the fact that $\partial p(s_1|s,a)/\partial a$ is Lipschitz with respect to $s$ and $a$ with constants $L_{ps}$ and $L_{pa}$ (cf., Assumption \ref{assumption_prob_dist}) we have that  
    \begin{equation}\label{eqn_deltap}
      \begin{split}
      \left\|\Delta p_t(s,s^{\prime},a,a^{\prime})\right\| \leq  \int \left(L_{ps}\left\|s^\prime-s^{\prime\prime}\right\|+L_{pa}\left\|a^\prime-a^{\prime\prime}\right\|\right)\\ \prod_{u=1}^{t-1}\pi_h(a_u|s_u)p(s_{u+1}|s_u,a_u) \, d\bbs_{t-1} d\bba_{t-1}.
        \end{split}
    \end{equation}
    %
    %
    %
        %
Using the previous bound and \eqref{eqn_aux_q_prime}, we can upper bound the norm of the difference $Q^\prime(s,a) - Q^\prime(s^\prime,a^\prime)$ as 
        \begin{equation}\label{eqn_aux_dif_Q}
          \begin{split}
            \left\| Q^\prime(s,a) - Q^\prime(s^\prime,a^\prime)\right\| \leq L_{r s}\left\|s-s^\prime \right\|+L_{r  a}\left\|a-a^\prime \right\|\\
            +\sum_{t=1}^\infty \gamma^t\int B_r\left\|\Delta p(s,s^\prime,a,a^\prime)\right\|\pi_h(a_t|s_t) \,d \bba_t d\bbs_t\\
            \end{split}
          \end{equation}
        Because $p(s_{u+1}|s_u,a_u)$ and $\pi_h(a_u|s_u)$ in \eqref{eqn_deltap} are density functions they integrate to one. Hence, the integral in the previous expression can be upper bounded by  
        \begin{equation}
          \begin{split}
            \int \left\|\Delta p_t(s,s^\prime,a,a^\prime)\right\|\pi_h(a_t|s_t) \,d \bba_t d\bbs_t  \\\leq \int L_{ps}\left\|s-s^{\prime}\right\|+L_{pa}\left\|a-a^{\prime}\right\| ds_1\\
            \leq |\ccalS| \left(L_{ps}\left\|s-s^{\prime}\right\|+L_{pa}\left\|a-a^{\prime}\right\|\right),
            \end{split}
        \end{equation}
        where $| \ccalS|$ is the measure of the set $\ccalS$. Then, one can further upper bound \eqref{eqn_aux_dif_Q} by
        \begin{equation}
          \begin{split}
            \left\| Q^\prime(s,a) - Q^\prime(s^\prime,a^\prime)\right\| \leq L_{r s}\left\|s-s^\prime \right\|+L_{r  a}\left\|a-a^\prime \right\|\\
            +\sum_{t=1}^\infty \gamma^t B_r|\ccalS| \left(L_{ps}\left\|s-s^{\prime}\right\|+L_{pa}\left\|a-a^{\prime}\right\|\right). 
            \end{split}
        \end{equation}
        Because the sum of the geometric yields $\gamma/(1-\gamma)$, it follows that  $Q^\prime(s,a)$ satisfies 
        \begin{equation}\label{eqn_q_prime_lipschitz}
          \left\| Q^\prime(s,a) - Q^\prime(s^\prime,a^\prime)\right\| \leq L_{Qs}\left\|s-s^\prime \right\|+L_{Qa}\left\|a-a^\prime \right\|,
        \end{equation}
        with $L_{Qs}=L_{rs}+\frac{\gamma B_r}{1-\gamma}L_{ps} |\ccalS|$ and $L_{Qa}=L_{ra}+\frac{\gamma B_r}{1-\gamma}L_{pa} |\ccalS|$. We next show that $h$ is Lipschitz. Using the reproducing property of the kernel, we can write the difference between $h(s)$ and $h(s^\prime)$ as 
    \begin{equation}
h(s)-h(s^\prime) = \left\langle h, \kappa(s,\cdot)-\kappa(s^\prime,\cdot)\right\rangle.
      \end{equation}
    Using the Cauchy-Schwartz inequality we can upper bound the previous inner product by 
    \begin{equation}
      \left\|h(s)-h(s^\prime)\right\| \leq \left\|h\right\| \sqrt{\kappa(s,s)+\kappa(s^\prime,s^\prime)-2\kappa(s,s^\prime)}
    \end{equation}
    Let us define $f(s,s^\prime) =  \sqrt{\kappa(s,s)+\kappa(s^\prime,s^\prime)-2\kappa(s,s^\prime)}$ and show that it is Lipschitz. To do so, use the following change of variables {$u=\Sigma_\ccalH^{-1/2}\left(s-s^\prime\right)$} and write $f(u)$ as follows
    \begin{equation}
      f(u) = \sqrt{2} \sqrt{1-e^{-\left\|u\right\|^2/2}}.
    \end{equation}
    Then, the gradient of $f(u)$ yields 
    \begin{equation}\label{eqn_gradient_aux_f}
      \nabla f(u) = \frac{1}{\sqrt{2}}\frac{e^{-\left\|u\right\|^2/2}u}{\sqrt{1-e^{-\left\|u\right\|^2/2}}}.
    \end{equation}
    Notice that the only point where the function might not be bounded is when $u=0$, since the limit of the denominator is zero. To show that this is not the case, observe that a second order Taylor approximation of that term yields
    \begin{equation}
      1-e^{-\left\|u\right\|^2/2} = \frac{1}{2}\left\|u\right\|^2 +o(\left\|u\right\|^2), 
    \end{equation}
    where $o(\left\|u\right\|^2)$ is a function such that $\lim_{\left\|u\right\| \to 0} o(\left\|u\right\|^2)/\left\|u\right\|^2=0$.  Thus we have that
    \begin{equation}
      \begin{split}
        \lim_{\left\|u\right\| \to 0} \left\|\nabla f(u)\right\| &= \lim_{\left\|u\right\|\to 0} \frac{e^{-\left\|u\right\|^2/2}\left\|u\right\|}{\left\|u\right\|}=1
        \end{split}
    \end{equation}
    It can be shown that the gradient of $\left\|\nabla f(u)\right\|$ is always differentiable except at $u=0$ and it never attains the value zero except for at the limit when $\left\|u\right\|\to \infty$. This means, that there are no critical points of $\left\|\nabla f(u)\right\|$ except at infinity. On the other hand it follows from \eqref{eqn_gradient_aux_f} that $\lim_{\left\|u\right\| \to \infty} \left\|\nabla f(u)\right\| =0$, so, the critical point at infinity is a minimum. Thus, the maximum norm of $\nabla f(u)$ is attained at $u=0$ and it takes the value $1$. Thus $f(u)$ is Lipschitz with constant $1$. Use the fact that $f(0)=0$ to bound
    \begin{equation}
      \begin{split}
        \left|f(u) \right| &\leq  \left\|u\right\|  = \left\|\Sigma_{\ccalH}^{-1/2}(s-s^\prime)\right\| \\ &
        \leq  \lambda_{\min}(\Sigma_\ccalH)^{-1/2}\left\|s-s^\prime\right\|.
        \end{split}
    \end{equation}
    %
The latter shows that $h(s)$ is Lipschitz with constant $L_h:=\left\|h\right\|\lambda_{\min}(\Sigma_\ccalH)^{-1/2}$.  We next use this result to complete the proof that $Q^\prime(s)$ is Lipschitz.  From its definition (cf., \eqref{eqn_q_prime_def}) and the fact that its Lipschitz (cf., \eqref{eqn_q_prime_lipschitz}) we have that
    \begin{equation}
      \begin{split}
        \left\|      Q^\prime(s)-Q^\prime(s^\prime)\right\| \leq L_{Q_s}\left\|s-s^\prime\right\| + L_{Q_a}\left\|h(s)-h(s^\prime)\right\|.
        \end{split}
    \end{equation}
    Because $h(s)$ is Lipschitz it follows that 
        \begin{equation}
      \begin{split}
        \left\| Q^\prime(s)-Q^\prime(s^\prime)\right\| \leq (L_{Q_s}+L_{Q_a}L_h)\left\|s-s^\prime\right\|.
        \end{split}
        \end{equation}
        This completes the proof of the first claim of the proposition. To show that $D(s)$ is bounded consider its norm. Using its expression in \eqref{eqn_D_change_variable} it is possible to uper bound it as
        \begin{equation}
\left\|D(s)\right\| \leq \mathbb{E}\left\| Q\left(s,h(s)+\Sigma^{1/2}\zeta\right)\zeta\right\|.
        \end{equation}
        Since $|Q(s,a)|<B_r/(1-\gamma)$ (cf., Lemma 3 \cite{paternain2018stochastic}) it follows that
        \begin{equation}
          \begin{split}
            \left\|D(s)\right\| &\leq \frac{B_r}{1-\gamma}\mathbb{E} \left\|\zeta\right\| = \frac{\sqrt{2}B_r}{1-\gamma}\frac{\Gamma\left(\frac{p+1}{2}\right)}{\Gamma\left(\frac{p}{2}\right)},            \end{split}
                \end{equation}
where $\Gamma$ is the Gamma function. 
  \end{proof}
To complete the proof of Lemma \ref{lemma_lipshitz} observe that we can write the difference
\begin{equation}
  \begin{split}
    \left\|D(s)\rho_{s_0}(s) - D(s^\prime)\rho_{s_0}(s^\prime)\right\| \leq \left\|D(s)\rho_{s_0}(s)-D(s)\rho_{s_0}(s^\prime)\right\|\\
    +\left\|D(s)\rho_{s_0}(s^\prime)-D(s^\prime)\rho_{s_0}(s^\prime)\right\|.
    \end{split}
\end{equation}
Using the Lipschitz continuity and the boundedness of both $D(s)$ and $\rho_{s_0}(s)$ it follows that\begin{equation}
  \begin{split}
    \left\|D(s)\rho_{s_0}(s) - D(s^\prime)\rho_{s_0}(s^\prime)\right\| \leq B_DL_p\left\|s-s^\prime\right\|\\
    +B_{\rho}L_D\left\|s^\prime-s\right\|.
  \end{split}
\end{equation} 
  This completes the proof of the lemma.

\bibliographystyle{ieeetr}
\bibliography{bib}

\begin{IEEEbiography}[{\includegraphics[width=1in,height=1.25in,clip,keepaspectratio]{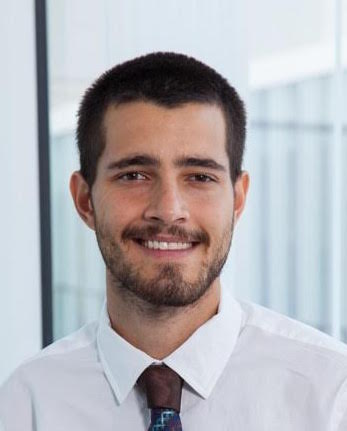}}]{Santiago Paternain} 
  received the B.Sc. degree in electrical engineering from Universidad de la República Oriental del Uruguay, Montevideo, Uruguay in 2012, the M.Sc. in Statistics from the Wharton School in 2018 and the Ph.D. in Electrical and Systems Engineering from the Department of Electrical and Systems Engineering, the University of Pennsylvania in 2018. He is currently an Assistant Professor in the Department of Electrical Computer and Systems Engineering at the Rensselear Polytechnic Institute. Prior to joining Rensselear, Dr. Paternain was a postdoctoral Researcher at the University of Pennsylvania. His research interests include optimization and control of dynamical systems. Dr. Paternain was the recipient of the 2017 CDC Best Student Paper Award and the 2019 Joseph and Rosaline Wolfe Best Doctoral Dissertation Award from the Electrical and Systems Engineering Department at the University of Pennsylvania.
\end{IEEEbiography}
\vspace{-1.5cm}
\begin{IEEEbiography}[{\includegraphics[width=1in,height=1.25in,clip,keepaspectratio]{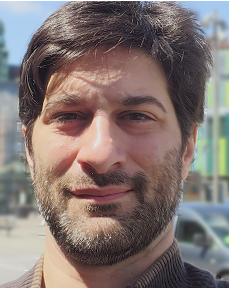}}]{Juan Andr\'es Bazerque} received the B.Sc. degree in electrical engineering from Universidad de la Rep\'ublica (UdelaR), Montevideo, Uruguay, in 2003, and the M.Sc. and Ph.D. degrees from the Department of Electrical and Computer Engineering, University of Minnesota (UofM), Minneapolis, in 2010 and 1013 respectively.
 Since 2015 he is an Assistant Professor with the  Department of Electrical Engineering  at UdelaR. His current research interests include  stochastic optimization and  networked systems, focusing on reinforcement learning, graph signal processing, and power systems optimization and control. 
Dr. Bazerque is the recipient of the UofM's Master Thesis Award 2009-2010, and co-reciepient of the best paper award at the 2nd International Conference on Cognitive Radio Oriented Wireless Networks and Communication 2007.   
 \end{IEEEbiography}
%
%
\begin{IEEEbiography}[{\includegraphics[width=1in,height=1.25in,clip,keepaspectratio]{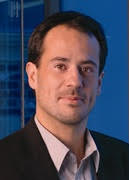}}]{Alejandro Ribeiro}  received the B.Sc. degree in electrical engineering from the Universidad de la Rep\'ublica Oriental del Uruguay, Montevideo, in 1998 and the M.Sc. and Ph.D. degree in electrical engineering from the Department of Electrical and Computer Engineering, the University of Minnesota, Minneapolis in 2005 and 2007. From 1998 to 2003, he was a member of the technical staff at Bellsouth Montevideo. After his M.Sc. and Ph.D studies, in 2008 he joined the University of Pennsylvania (Penn), Philadelphia, where he is currently the Rosenbluth Associate Professor at the Department of Electrical and Systems Engineering. His research interests are in the applications of statistical signal processing to the study of networks and networked phenomena. His focus is on structured representations of networked data structures, graph signal processing, network optimization, robot teams, and networked control. Dr. Ribeiro received the 2014 O. Hugo Schuck best paper award, and paper awards at the 2016 SSP Workshop, 2016 SAM Workshop, 2015 Asilomar SSC Conference, ACC 2013, ICASSP 2006, and ICASSP 2005. His teaching has been recognized with the 2017 Lindback award for distinguished teaching and the 2012 S. Reid Warren, Jr. Award presented by Penn's undergraduate student body for outstanding teaching. Dr. Ribeiro is a Fulbright scholar class of 2003 and a Penn Fellow class of 2015.
\end{IEEEbiography}

\end{document}